\documentclass[twoside]{article}

\usepackage{amsmath}
\usepackage{amsfonts}
\usepackage{amsthm}
\usepackage{bm}
\usepackage{booktabs} 
\usepackage[makeroom]{cancel}
\usepackage[inline]{enumitem}
\usepackage{graphicx}
\usepackage[dvipsnames]{xcolor} 

\usepackage{hyperref}

\usepackage[capitalise]{cleveref}

\hypersetup{
    linkcolor=blue,
    filecolor=magenta,
    urlcolor=cyan,
    pdfpagemode=FullScreen,
}

\newlist{enuminline}{enumerate*}{1}
\setlist[enuminline]{label=\textbf{\arabic*})}

\DeclareMathOperator*{\argmax}{arg\,max}

\DeclareMathOperator{\vect}{vec}
\DeclareMathOperator{\unvec}{unvec}
\DeclareMathOperator{\tril}{tril}
\DeclareMathOperator{\triu}{triu}
\DeclareMathOperator{\diag}{diag}
\DeclareMathOperator{\diff}{d}
\DeclareMathOperator{\tr}{tr}
\DeclareMathOperator{\chol}{chol}
\DeclareMathOperator{\KL}{KL}

\newtheorem{proposition}{Proposition}

\usepackage[accepted]{aistats2026}
\usepackage[round]{natbib}

\begin{document}

%

%

\twocolumn[

\aistatstitle{Inverse-Free Sparse Variational Gaussian Processes}

\aistatsauthor{ Stefano Cortinovis \And Laurence Aitchison \And Stefanos Eleftheriadis \And  Mark van der Wilk }

\aistatsaddress{ University of Oxford \And  University of Bristol \And No affiliation \And University of Oxford } ]

\begin{abstract}
  Gaussian processes (GPs) offer appealing properties but are costly to train at scale.
  Sparse variational GP (SVGP) approximations reduce cost yet still rely on Cholesky decompositions of kernel matrices, ill-suited to low-precision, massively parallel hardware.
  While one can construct valid variational bounds that rely only on matrix multiplications (matmuls) via an auxiliary matrix parameter, optimising them with off-the-shelf first-order methods is challenging.
  We make the inverse-free approach practical by proposing a better-conditioned bound and deriving a matmul-only natural-gradient update for the auxiliary parameter, markedly improving stability and convergence.
  We further provide simple heuristics, such as step-size schedules and stopping criteria, that make the overall optimisation routine fit seamlessly into existing workflows.
  Across regression and classification benchmarks, we demonstrate that our method
  \begin{enuminline}
    \item serves as a drop-in replacement in SVGP-based models (e.g., deep GPs),
    \item recovers similar performance to traditional methods, and
    \item can be faster than baselines when well tuned.
  \end{enuminline}
\end{abstract}

\section{INTRODUCTION}\label{sec:introduction}
Gaussian processes (GPs) \citep{williams2006gaussian} are flexible Bayesian priors over functions, valued for uncertainty that takes into account infinite basis functions, and automatic model selection.
These benefits come with a cubic training cost $\mathcal{O}(N^3)$ in the number of data points $N$, which limits their use at scale.
Sparse variational approximations \citep{titsias2009variational} replace full GP inference with a variational posterior over $M$ pseudo-inputs, called \textit{inducing points}, reducing the cost of full-batch training to $\mathcal{O}(NM^2 + M^3)$, so that the only cubic dependence is on $M$ rather than $N$.
Stochastic optimisation \citep{hensman2013gaussian} further reduces the per-iteration cost to $\mathcal{O}(BM^2 + M^3)$ with mini-batches of size $B \ll N$.

While these advances enable arbitrarily accurate approximations with \(M \ll N\) \citep{burt2019rates,burt2020convergence}, they still require computing the inverse and determinant of an \(M\times M\) kernel matrix.
In practice, these are obtained via Cholesky decompositions, which rely on inherently sequential linear algebra and high-precision arithmetic, poorly matched to modern accelerators optimised for low-precision, massively parallel matrix multiplications (matmuls).
To avoid decompositions, \citet{vdwilk2020variational,vdwilk2022improved} derived an inverse-free variational bound that replaces them with matmuls by introducing an auxiliary parameter $\mathbf T \in \mathbb{R}^{M\times M}$ whose optimum recovers the required inverse, but this objective can be difficult to optimise with off-the-shelf first-order methods such as Adam \citep{kingma2014adam}, leading to instability and slow convergence (\cref{fig:snelson_banana_trace}).
\begin{figure*}
    \centering
    \includegraphics[width=1.0\textwidth]{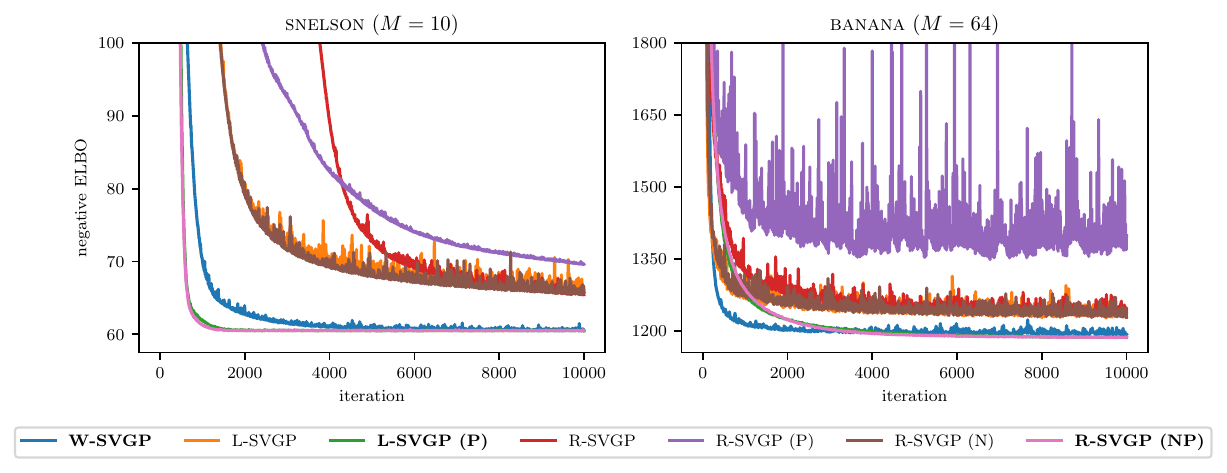}
    \caption{Loss traces on \textsc{snelson} and \textsc{banana} datasets. The \textsc{N} and \textsc{P} suffixes refer to the use of NG updates (\cref{sec:natgrad}) and inducing mean preconditioning (\cref{sec:preconditioning}), respectively. R-SVGP (NP) is the only inverse-free variant that matches the performance of W-SVGP and L-SVGP (P), which use Cholesky decompositions.}
    \label{fig:snelson_banana_trace}
\end{figure*}

In this work, we make the inverse-free approach practical (see \cref{fig:snelson_banana_trace}).
First, we show that, even with optimally tuned $\mathbf{T}$, a preconditioner on a variational parameter is needed to match the performance of standard SVGP methods.
Second, we derive a natural-gradient (NG) update for $\mathbf{T}$ that uses only matmuls, substantially improving stability and convergence.
Finally, we discuss a number of strategies, such as simple schedules and stopping rules, that make the $\mathbf{T}$ updates efficient and automatic.
The resulting SVGP objective
\begin{enuminline}
    \item can be \textit{evaluated} in quadratic time,
    \item serves as a drop-in replacement within complex SVGP-based models, and
    \item trains stably using only matmuls.
\end{enuminline}
While \textit{training} remains cubic in \(M\) due to the NG update for $\mathbf{T}$, the overall procedure is well-suited to parallelisation and low-precision arithmetic, as it only involves matmuls.
We evaluate our method on regression and classification benchmarks, where it achieves comparable performance to baselines that rely on Cholesky decompositions.

The remainder of the paper is organised as follows.
\Cref{sec:svgp} reviews SVGP parameterisations and the inverse-free bound of \citet{vdwilk2022improved}.
\Cref{sec:improving} presents our $\mathbf T$-tailored updates, preconditioning, and practical training strategies.
\Cref{sec:results} evaluates the proposed method on a variety of benchmarks.
\Cref{sec:related} briefly surveys related work, and \Cref{sec:discussion} discusses limitations and extensions.

\section{SVGP PARAMETERISATIONS}\label{sec:svgp}
We consider\footnote{The SVGP framework is presented for a shallow GP prior with scalar output, but the ideas presented here extend naturally to multi-output \citep{vdw2020framework} and deep GP models \citep{salimbeni2017doubly}.} the problem of learning a function $f\!:\! \mathbb{R}^D\!\to\! \mathbb{R}$ through Bayesian inference, with a prior $f(\cdot) \!\sim\! \mathcal{GP}(0, k_{\bm{\psi}}(\cdot, \cdot))$ and an arbitrary factorised likelihood with density $p(\mathbf y | \mathbf f) = \prod_{n=1}^N p(y_n | f(\mathbf x_n); \bm{\eta})$, where $\bm{\theta} = \{\bm{\psi}, \bm{\eta}\}$ are model hyperparameters\footnote{In the sequel, dependence of $k_{\bm{\psi}}(\cdot, \cdot)$ and $p(y_n | f(\mathbf{x_n}; \bm{\eta}))$ on $\bm{\psi}$ and $\bm{\eta}$ is suppressed for brevity.}.
To make predictions, we need to approximate the maximum marginal likelihood hyperparameters ${\bm{\theta}}^* = \argmax_{\bm{\theta}} \log p(\mathbf y|\bm{\theta})$, and the posterior $p(f | \mathbf{y})$.
To do so while avoiding the cost of full GP inference, one popular approach is to use sparse variational Gaussian processes (SVGPs) \citep{hensman2013gaussian}, which select an approximate posterior distribution $q(f)$ by minimising the KL-divergence to the true posterior $p(f | \mathbf{y})$.
The approximate posterior is constructed by conditioning the prior on $M \!\ll\! N$ inducing points \citep{quinonero2005unifying}, at input locations $\mathbf{Z}\!\in\!\mathbb{R}^{M \times D}$ and output values $\mathbf{u}\!=\! f(\mathbf{Z})$, resulting in a variational posterior $q(f) = \int p(f_{\neq \mathbf{u}} | \mathbf{u}) q(\mathbf{u}) \mathrm{d} \mathbf u$ (see  \citet{matthewsthesis,vdwthesis} for further details).
Following variational inference \citep{blei2017variational}, $\KL[q(f)\vert\vert p(f|\mathbf{y})]$ is minimised by maximising the evidence lower bound (ELBO) with respect to $\bm{\theta}$, $\mathbf Z$ and the parameters of $q(\mathbf u)$.
When $q(\mathbf{u})$ is chosen to be Gaussian, the ELBO becomes
\begin{equation}
    \mathcal{L} = \textstyle{\sum}_{n = 1}^N \mathbb{E}_{\mathcal{N}(\mu_n, \sigma_n^2)} \left[\log p(y_n | f(\mathbf{x}_n))\right] - \KL\left[q(\mathbf{u}) || p(\mathbf{u}) \right], \label{eq:elbo}
\end{equation}
where $\mu_n$ and $\sigma_n$ are, respectively, the predictive mean and variance at the input location $\mathbf{x}_n$.
When optimised over mini-batches of size $B$, the SVGP ELBO \eqref{eq:elbo} has time complexity $\mathcal{O}(B M^2 + M^3)$ and can be used to scale GP inference to large datasets under arbitrary likelihoods, mirroring stochastic optimisation in neural networks.
However, unlike neural networks, the computation of $\mu_n$, $\sigma_n^2$ and $\KL[q(\mathbf{u})||p(\mathbf{u})]$ involves a Cholesky decomposition, which is ill-suited to modern deep-learning hardware.
Furthermore, the choice of the parameterisation of the Gaussian distribution $q(\mathbf{u})$ can have a significant impact on the optimisation stability and the tightness of the ELBO.
Below, we briefly review three common parameterisations, as well as their relationship with the inverse-free bound of \citet{vdwilk2022improved}, highlighting in \textcolor{red}{red} the terms that involve matrix decompositions.

\paragraph{Marginal Parameterisation (M-SVGP).}
The most common freeform $q(\mathbf{u}) = \mathcal{N}(\mathbf{m}, \mathbf{S})$ results in
{
    \allowdisplaybreaks
    \begin{align}
        \mu_n &= \mathbf{k}_{n \mathbf{u}} \textcolor{red}{\mathbf{K}_{\mathbf{u} \mathbf{u}}^{-1}} \mathbf{m} \nonumber \\
        \sigma_n^2 &= k_{nn} - \mathbf{k}_{n \mathbf{u}} \textcolor{red}{\mathbf{K}_{\mathbf{u} \mathbf{u}}^{-1}} \mathbf{k}_{\mathbf{u} n} + \mathbf{k}_{n \mathbf{u}} \textcolor{red}{\mathbf{K}_{\mathbf{u} \mathbf{u}}^{-1}} \mathbf{S} \textcolor{red}{\mathbf{K}_{\mathbf{u} \mathbf{u}}^{-1}} \mathbf{k}_{\mathbf{u} n}, \nonumber \\
        \KL&\left[q(\mathbf{u}) || p(\mathbf{u}) \right] = \frac{1}{2}\left(\text{tr}(\textcolor{red}{\mathbf{K}_{\mathbf{u} \mathbf{u}}^{-1}} \mathbf{S}) + \mathbf{m}^\top \textcolor{red}{\mathbf{K}_{\mathbf{u} \mathbf{u}}^{-1}} \mathbf{m} - M\right. \nonumber  \\
        &\qquad\qquad\qquad\qquad\left.+ \textcolor{red}{\log |\mathbf{K}_{\mathbf{u} \mathbf{u}}|} - \log |\mathbf{S}|\right), \nonumber
    \end{align}
}%
where $\mathbf{K}_{\mathbf{u} \mathbf{u}} = k(\mathbf{Z}, \mathbf{Z})$, $\mathbf{k}_{\mathbf{u} n} = \mathbf{k}_{n \mathbf{u}}^\top = k(\mathbf{Z}, \mathbf{x}_n)$, and $k_{nn} = k(\mathbf{x}_n, \mathbf{x}_n)$.

\paragraph{Whitened Parameterisation (W-SVGP).}
\citet{hensman2015mcmc} propose to reparameterise $q(\mathbf{u}) = \mathcal{N}(\mathbf{m}, \mathbf{S})$ with $\mathbf{m} = \mathbf{L}_{\mathbf{u}\mathbf{u}} \tilde{\mathbf{m}}$ and $\mathbf{S} = \mathbf{L}_{\mathbf{u}\mathbf{u}} \tilde{\mathbf{S}} \mathbf{L}_{\mathbf{u}\mathbf{u}}^\top$, where $\mathbf{L}_{\mathbf{u}\mathbf{u}} = \chol(\mathbf{K}_{\mathbf{u}\mathbf{u}})$ is the Cholesky factor of $\mathbf{K}_{\mathbf{u}\mathbf{u}}$ and $\tilde{\mathbf{S}}$ is PD.
This results in
\begin{align}
    \mu_n &= \mathbf{k}_{n \mathbf{u}} \textcolor{red}{\mathbf{L}_{\mathbf{u} \mathbf{u}}^{-\top}} \tilde{\mathbf{m}} \nonumber \\
    \sigma_n^2 &= k_{nn} - \mathbf{k}_{n \mathbf{u}} \textcolor{red}{\mathbf{K}_{\mathbf{u} \mathbf{u}}^{-1}} \mathbf{k}_{\mathbf{u} n} + \mathbf{k}_{n \mathbf{u}} \textcolor{red}{\mathbf{L}_{\mathbf{u} \mathbf{u}}^{-\top}} \tilde{\mathbf{S}} \textcolor{red}{\mathbf{L}_{\mathbf{u} \mathbf{u}}^{-1}} \mathbf{k}_{\mathbf{u} n}, \nonumber \\
    \KL&\left[q(\mathbf{u}) || p(\mathbf{u}) \right] = \KL[\mathcal{N}(\tilde{\mathbf{m}}, \tilde{\mathbf{S}}) || \mathcal{N}(\mathbf{0}, \mathbf{I})]. \nonumber
\end{align}
This is a form of whitening of the inducing variable $\mathbf{u}$, which often improves optimisation and is used as the default SVGP parameterisation in popular GP packages \citep[e.g.,][]{matthews2017gpflow}.

\paragraph{Likelihood Parameterisation (L-SVGP).}
\citet{panosFullyScalableGaussian2018} suggest to reparameterise $\mathbf{m} = \mathbf{K}_{\mathbf{u} \mathbf{u}} \tilde{\mathbf{m}}$ and $\mathbf{S} = (\mathbf{K}_{\mathbf{u}\mathbf{u}}^{-1} + \tilde{\mathbf{S}}^{-1})^{-1}$, where $\tilde{\mathbf{S}}$ is PD diagonal. This results in
\begin{align}
    &\mu_n = \mathbf{k}_{n \mathbf{u}} \tilde{\mathbf{m}}, \quad \sigma_n^2 = k_{nn} - \mathbf{k}_{n \mathbf{u}} \textcolor{red}{\tilde{\mathbf{K}}^{-1}} \mathbf{k}_{\mathbf{u} n} =: \sigma_n^{2(\text{L})}, \label{eq:pv} \\
    &\KL\left[q(\mathbf{u}) || p(\mathbf{u})\right] = \frac{1}{2}\left(-\text{tr}(\textcolor{red}{\tilde{\mathbf{K}}^{-1}} \mathbf{K}_{\mathbf{u} \mathbf{u}}) + \tilde{\mathbf{m}}^\top \mathbf{K}_{\mathbf{u} \mathbf{u}} \tilde{\mathbf{m}} \right. \nonumber \\
    &\qquad\qquad\qquad\qquad\quad\ \ \left.+ \textcolor{red}{\log |\tilde{\mathbf{K}}|} - \log |\tilde{\mathbf{S}}|\right), \nonumber
\end{align}
where $\tilde{\mathbf{K}} = \mathbf{K}_{\mathbf{u} \mathbf{u}} + \tilde{\mathbf{S}}$. Inverting $\tilde{\mathbf{K}}$ instead of $\mathbf K$ is more stable without the need for \emph{jitter} terms, due to having lower-bounded minimum eigenvalue thanks to $\tilde{\mathbf{S}}$.
Moreover, further manipulations of the L-SVGP lead to an inverse-free bound, as we discuss next.

\paragraph{Inverse-Free Parameterisation (R-SVGP).} \Citet{vdwilk2022improved} develop an inverse-free bound for L-SVGP, starting by upper bounding the predictive variance \eqref{eq:pv} as
\begin{equation}
    \sigma^{2(\text{L})}_n \leq k_{nn} - \mathbf{k}_{n \mathbf{u}} (2 \mathbf{T} - \mathbf{T} \tilde{\mathbf{K}} \mathbf{T}) \mathbf{k}_{\mathbf{u}n} =: U_n,  \label{eq:lsvgp_ub}
\end{equation}
where $\mathbf{T} \in \mathbb{R}^{M \times M}$, and with equality when $\mathbf{T} = \tilde{\mathbf{K}}^{-1}$. Working backwards, they then find that, if the variance $\mathbf{S}$ of M-SVGP is reparameterised as $\mathbf{S} = \mathbf{K}_{\mathbf{u} \mathbf{u}} - \mathbf{K}_{\mathbf{u} \mathbf{u}} (2\mathbf{T} -\mathbf{T}\tilde{\mathbf{K}}\mathbf{T})\mathbf{K}_{\mathbf{u}\mathbf{u}}$, the predictive variance matches the upper bound, i.e., $\sigma_n^2 = U_n$.
This improves over the previous inverse-free bound for M-SVGP \citep{vdwilk2020variational} by allowing for a closed-form inverse-free bound for the KL term
\begin{align}
    \KL&\left[q(\mathbf{u}) || p(\mathbf{u})\right] \leq \frac{1}{2}\left(-\text{tr}((2\mathbf{T} - \mathbf{T} \tilde{\mathbf{K}} \mathbf{T}) \mathbf{K}_{\mathbf{u} \mathbf{u}}) - M \right. \nonumber \\
    &\left. + \tilde{\mathbf{m}}^\top \mathbf{K}_{\mathbf{u} \mathbf{u}} \tilde{\mathbf{m}} + \text{tr}(\tilde{\mathbf{K}} \mathbf{T}) - \log |\mathbf{T}| - \log |\tilde{\mathbf{S}}|\right), \label{eq:kl_rsvgp}
\end{align}
with equality when $\mathbf{T} = \tilde{\mathbf{K}}^{-1}$.
Putting all this together leads to a \textit{relaxed} bound that
\begin{enuminline}
    \item is a valid ELBO,
    \item depends on the additional parameter $\mathbf{T}$,
    \item is free of matrix decompositions,
    \item recovers the L-SVGP solution when $\mathbf T = \tilde{\mathbf{K}}^{-1}$.
\end{enuminline}

\paragraph{Alternative Inverse-Free Parameterisations.}
While we do not pursue the idea further in this work, we note that the R-SVGP construction, whereby one exploits an inverse-free upper bound to the predictive variance, is more general and may be applied to parameterisations other than L-SVGP.
For instance, starting from M-SVGP, reparameterise $\mathbf{m} = \mathbf{K}_{\mathbf{u}\mathbf{u}} \mathbf{R} \tilde{\mathbf{m}}$ and $\mathbf{S} = \mathbf{K}_{\mathbf{u}\mathbf{u}} \mathbf{R} \tilde{\mathbf{S}} \mathbf{R}^\top \mathbf{K}_{\mathbf{u}\mathbf{u}}$, where $\tilde{\mathbf{S}}$ is PD and $\mathbf{R}$ is lower triangular with positive diagonal.
This results in
\begin{equation*}
    \mu_n = \mathbf{k}_{n \mathbf{u}} \mathbf{R} \tilde{\mathbf{m}}, \quad \sigma_n^2 = k_{nn} - \mathbf{k}_{n \mathbf{u}} (\mathbf{K}_{\mathbf{u}\mathbf{u}}^{-1} - \mathbf{R}\tilde{\mathbf{S}}\mathbf{R}^\top) \mathbf{k}_{\mathbf{u} n}.
\end{equation*}
Similarly to \cref{eq:lsvgp_ub}, the predictive variance can be upper bounded with
\begin{equation*}
    \sigma_n^2 \leq k_{nn} + \mathbf{k}_{n \mathbf{u}} \mathbf{R} (\mathbf{R}^\top\mathbf{K}_{\mathbf{u}\mathbf{u}}\mathbf{R} - 2\mathbf{I} + \tilde{\mathbf{S}})\mathbf{R}^\top \mathbf{k}_{\mathbf{u} n},
\end{equation*}
with equality when $\mathbf{R} = \chol(\mathbf{K}_{\mathbf{u}\mathbf{u}}^{-1})$.
Working backwards, the upper bound above is equal to the predictive variance induced by the reparameterisation $\mathbf{S} = \mathbf{K}_{\mathbf{u}\mathbf{u}} + \mathbf{K}_{\mathbf{u}\mathbf{u}}\mathbf{R}(\mathbf{R}^\top \mathbf{K}_{\mathbf{u}\mathbf{u}} \mathbf{R} - 2 \mathbf{I} + \tilde{\mathbf{S}})\mathbf{R}^\top\mathbf{K}_{\mathbf{u}\mathbf{u}}$. The KL term implied by this choice of $\mathbf{m}$ and $\mathbf{S}$ admits the upper bound
\begin{align*}
    &\KL[q(\mathbf{u}) || p(\mathbf{u})] \leq \frac{1}{2}\left(\tilde{\mathbf{m}}^\top \mathbf{R}^\top \mathbf{K}_{\mathbf{u}\mathbf{u}} \mathbf{R} \tilde{\mathbf{m}} - \log|\tilde{\mathbf{S}}| \right. \nonumber \\
    &\quad\left. + \mathrm{tr}\left((\tilde{\mathbf{S}} - 3 \mathbf{I} + \mathbf{R}^\top \mathbf{K}_{\mathbf{u}\mathbf{u}} \mathbf{R})\mathbf{R}^\top \mathbf{K}_{\mathbf{u}\mathbf{u}} \mathbf{R}\right) + M \right) \nonumber \\
    &\quad+ \KL\left[\mathcal{N}(\mathbf{0}, \mathbf{R}\mathbf{R}^\top) || \mathcal{N}(\mathbf{0}, \textcolor{red}{\mathbf{K}_{\mathbf{u}\mathbf{u}}^{-1}})\right],
\end{align*}
with equality when $\mathbf{R} = \chol(\mathbf{K}_{\mathbf{u}\mathbf{u}}^{-1})$.
The resulting ELBO depends on the additional parameter $\mathbf{R}$ and, at the optimum $\mathbf{R} = \chol(\mathbf{K}_{\mathbf{u}\mathbf{u}}^{-1})$, recovers a whitened bound, as in W-SVGP, where the whitening is performed via $\chol(\mathbf{K}_{\mathbf{u}\mathbf{u}}^{-1})$ rather than $\mathbf{L}_{\mathbf{u}\mathbf{u}}^{-\top}$.
While the bound above still contains an inverse in the inner KL term, we conjecture that the techniques developed in \cref{sec:natgrad} for the optimisation of $\mathbf{T}$ in R-SVGP may be used to efficiently eliminate it from the bound, thereby making the required update for the rest of the parameters effectively inverse-free.
A brief derivation of the bound above is provided in \cref{app:alternative_rsvgp}.

\section{IMPROVED INVERSE-FREE BOUND}\label{sec:improving}
When $\mathbf{T}$ is optimised to $\tilde{\mathbf{K}}^{-1}$, the R-SVGP bound coincides with the L-SVGP bound.
In practice, R-SVGP is competitive when
\begin{enuminline}
    \item $\mathbf{T}$ is kept close to its optimum throughout training, and
    \item L-SVGP itself can reach competitive performance.
\end{enuminline}
In \citet{vdwilk2022improved}, $\mathbf{T}$ is parameterised via its Cholesky factor $\mathbf{L}$ with $\mathbf{T}=\mathbf{L}\mathbf{L}^\top$ and updated jointly with Adam, but this leads to poor optimisation even on simple datasets (\cref{fig:snelson_banana_trace}).
\Cref{sec:natgrad} addresses this with a tailored optimiser for $\mathbf{T}$.
Moreover, we find that L-SVGP, as presented in \cref{sec:svgp}, exhibits slower convergence than standard W-SVGP, often leading to worse predictive performance.
\Cref{sec:preconditioning} addresses this by introducing a preconditioner for the L-SVGP inducing mean, together with an inverse-free analogue for R-SVGP.
Finally, \cref{sec:stopping,sec:practical} provide practical strategies to make R-SVGP optimisation efficient and automatic.

\subsection{Iterative Inversion via Natural Gradients}\label{sec:natgrad}
Our goal is to keep $\mathbf{T}$ close to $\tilde{\mathbf{K}}^{-1}$ without ever performing a matrix decomposition.
We parameterise $\mathbf{T}$ via its Cholesky factor $\mathbf{L}$, so that the problem reduces to finding the Cholesky factor of $\tilde{\mathbf{K}}^{-1}$, $\mathbf{L}_{\tilde{\mathbf{K}}^{-1}}$.

\paragraph{Auxiliary Inversion Objective.}
For any symmetric PD matrix $\mathbf A$, the Cholesky factor of its inverse, $\mathbf L_{\mathbf A^{-1}}$, is the unique minimiser of the KL divergence between two zero-mean Gaussians:
\begin{align}
  \mathbf L_{\mathbf A^{-1}} &= \arg\min_{\mathbf L}\, \ell_{\mathbf A}(\mathbf L), \nonumber\\
  \ell_{\mathbf A}(\mathbf L) &:= \KL\!\left[\mathcal N(\mathbf 0,\mathbf L\mathbf L^\top)\,\|\,\mathcal N(\mathbf 0,\mathbf A^{-1})\right]. \label{eq:natgrad_loss}
\end{align}

\paragraph{Natural Gradient Updates.}
Natural-gradient (NG) updates are well-suited to KL objectives \citep{amari1998natural}.
Let $\mathbf F$ be the Fisher information of $\mathcal N(\mathbf 0,\mathbf L\mathbf L^\top)$ with respect to $\mathbf L$.
The NG is $\tilde{\nabla}\ell_{\mathbf A}=\mathbf F^{-1}\nabla\ell_{\mathbf A}$ and, for \cref{eq:natgrad_loss}, admits a closed form, derived in \cref{app:natgrad}.
\begin{proposition}\label{prop:natgrad}
    Let $\ell(\mathbf{L})$ be as in \cref{eq:natgrad_loss}. Then,
    \begin{equation}
        \tilde{\nabla}\ell_{\mathbf A}
        = \mathbf L\!\left[\operatorname{tril}(\mathbf L^\top \mathbf A \mathbf L) - \tfrac12\!\left(\mathbf I + \operatorname{diag}(\mathbf L^\top \mathbf A \mathbf L)\right)\right], \label{eq:natgrad}
    \end{equation}
    where $\tril(\cdot)$ and $\diag(\cdot)$ return the lower triangular part and the diagonal of a matrix. 
\end{proposition}

\paragraph{Alternating Optimisation.}
We optimise the R-SVGP bound wrt $\mathbf L$ and $\bm\xi=\{\bm\theta,\mathbf Z,\tilde{\mathbf m},\tilde{\mathbf S}\}$ by alternating:
\begin{enumerate}[label=\textbf{\arabic*)}]\setlength\itemsep{0em}
  \item \textbf{NG step:} given $\tilde{\mathbf K}$, update $\mathbf L \leftarrow \mathbf L - \gamma\,\tilde{\nabla}_{\mathbf L}\ell_{\tilde{\mathbf K}}$ for $t^\star$ steps;
  \item \textbf{Adam step:} given $\mathbf L$, update $\bm\xi$ with an Adam step on the R-SVGP ELBO.
\end{enumerate}
The NG expression \eqref{eq:natgrad} avoids matrix decompositions, keeping the procedure \textit{inverse-free}.
The update represents an $\mathbf{L}$-space version of the classical Newton-Schulz iteration for matrix inversion \citep{ben1965iterative}, which is itself a NG step on the objective \eqref{eq:natgrad_loss} in $\mathbf{T}$-space.
As a result, by standard Newton-type arguments \citep[e.g.,][Chapter 5]{kelley1995iterative}, the auxiliary $\mathbf{L}$-iteration enjoys local quadratic convergence under regularity conditions.

\subsection{Inducing Mean Preconditioning}\label{sec:preconditioning}
We attribute the suboptimal optimisation of L-SVGP to the inducing-mean parameterisation $\mathbf m=\mathbf K_{\mathbf u\mathbf u}\tilde{\mathbf m}$, and introduce a preconditioner $\mathbf P\in\mathbb R^{M\times M}$ via $\mathbf m=\mathbf K_{\mathbf u\mathbf u}\mathbf P\tilde{\mathbf m}$.
For L-SVGP, we take $\mathbf P^{(\mathrm L)}=\tilde{\mathbf K}^{-1}=(\mathbf K_{\mathbf u\mathbf u}+\tilde{\mathbf S})^{-1}$.
When $\mathbf X=\mathbf Z$, this yields $\mu_n=\mathbf k_{n\mathbf u}\tilde{\mathbf K}^{-1}\tilde{\mathbf m}$, matching the form of the GP posterior mean under a Gaussian likelihood (with optimal $\tilde{\mathbf m}=\mathbf y$, $\tilde{\mathbf S}=\sigma^2_{\text{obs}}\mathbf I$ independent of other model parameters).

The same $\tilde{\mathbf K}^{-1}$ is incompatible with R-SVGP as it reintroduces inversions.
Instead, we use the inverse-free analogue $\mathbf P^{(\mathrm R)}=2\mathbf T-\mathbf T\tilde{\mathbf K}\mathbf T$, which recovers the \emph{same} ELBO gradients as preconditioned L-SVGP when $\mathbf T$ is held at $\tilde{\mathbf K}^{-1}$ (proof in \cref{app:grad}):
\begin{proposition}\label{prop:grad}
  Let $\tilde{\mathcal L}_{\text{lsvgp}}$ and $\tilde{\mathcal L}_{\text{rsvgp}}$ denote the L-SVGP and R-SVGP bounds with $\mathbf P^{(\mathrm L)}=\tilde{\mathbf K}^{-1}$ and $\mathbf P^{(\mathrm R)}=2\mathbf T-\mathbf T\tilde{\mathbf K}\mathbf T$, respectively, and parameters $\bm\xi=\{\bm\theta,\mathbf Z,\tilde{\mathbf m},\tilde{\mathbf S}\}$. Then
  \[
    \nabla_{\bm\xi}\tilde{\mathcal L}_{\text{lsvgp}}=\nabla_{\bm\xi}\tilde{\mathcal L}_{\text{rsvgp}}
  \]
  when $\mathbf T=\operatorname{stopgrad}(\tilde{\mathbf K}^{-1})$. The equivalence does not hold for the naive choice $\mathbf P=\mathbf T$.
\end{proposition}

Summarising, substituting $\mathbf m=\mathbf K_{\mathbf u\mathbf u}\mathbf P^{(\mathrm R)}\tilde{\mathbf m}$ and $\mathbf S=\mathbf K_{\mathbf u\mathbf u}-\mathbf K_{\mathbf u\mathbf u}\mathbf P^{(\mathrm R)}\mathbf K_{\mathbf u\mathbf u}$ into the M-SVGP form and using the KL bound \eqref{eq:kl_rsvgp} yields
\begin{align}
  &\mu_n=\mathbf k_{n\mathbf u}\mathbf P^{(\mathrm R)}\tilde{\mathbf m},\quad
  \sigma_n^2=k_{nn}-\mathbf k_{n\mathbf u}\mathbf P^{(\mathrm R)}\mathbf k_{\mathbf u n}, \nonumber\\
  &\KL\!\left[q(\mathbf u)\,\|\,p(\mathbf u)\right]\le \tfrac12\!\left(-\operatorname{tr}(\mathbf P^{(\mathrm R)}\mathbf K_{\mathbf u\mathbf u})+\operatorname{tr}(\tilde{\mathbf K}\mathbf T)-M\right. \nonumber\\
  &\quad\left.+\,\tilde{\mathbf m}^\top \mathbf P^{(\mathrm R)}\mathbf K_{\mathbf u\mathbf u}\mathbf P^{(\mathrm R)}\tilde{\mathbf m}-\log|\mathbf T|-\log|\tilde{\mathbf S}|\right).
  \label{eq:rsvgp_precond}
\end{align}
As will be shown in \cref{sec:results}, the proposed preconditioning allows R-SVGP (and L-SVGP) to reach performance comparable to W-SVGP (\cref{fig:snelson_banana_trace}).

\paragraph{Remark (connection to Newton-Schulz).}
The preconditioner $\mathbf P^{(\mathrm R)}=2\mathbf T-\mathbf T\tilde{\mathbf K}\mathbf T$, which appears several times in both the standard \eqref{eq:kl_rsvgp} and preconditioned \eqref{eq:rsvgp_precond} R-SVGP bound, is exactly one Newton-Schulz iteration for $\tilde{\mathbf K}^{-1}$ starting from $\mathbf T$, which is also a NG step on the objective \eqref{eq:natgrad_loss} in $\mathbf{T}$-space.
Thus, R-SVGP optimisation effectively alternates NG updates in $\mathbf{L}$-space with a single implicit NG step in $\mathbf{T}$-space as part of the bound computation, unveiling an insightful connection between inverse-free variational bounds and iterative matrix inversion.

\subsection{Stopping Criteria}\label{sec:stopping}
As part of the alternating optimisation routine proposed in \cref{sec:natgrad}, we wish to choose the minimum number $t^*$ of NG updates per Adam step that brings R-SVGP close enough to L-SVGP.
Below we discuss two possible stopping criteria to choose $t^*$ adaptively.

A first option is to monitor when the normalised residual
\begin{equation}
    || \mathbf{L}_t^\top \tilde{\mathbf{K}} \mathbf{L}_t - \mathbf{I}||_{\mathrm{F}} / \sqrt{M} \label{eq:stopping_frobenius}
\end{equation}
falls below a threshold $\epsilon$.
This quantity is virtually free, as $\mathbf{L}_t^\top \tilde{\mathbf{K}} \mathbf{L}_t$ is reused in the NG computation \eqref{eq:natgrad}, and it is meaningful, as it is zero iff $\mathbf{L}_t = \chol(\tilde{\mathbf{K}}^{-1})$.
Moreover, this stopping criterion is general and can be applied to arbitrary likelihoods for both shallow and deep GPs.
In practice, we find that a tolerance of $\epsilon = 10^{-3}$ to $5 \times 10 ^{-3}$ works well for a wide range of shallow GP prediction tasks.

Alternatively, for Gaussian likelihoods, one can directly monitor the slack introduced in the R-SVGP bound wrt L-SVGP by the upper bound \eqref{eq:lsvgp_ub} on the predictive variance $\sigma^{2(\text{L})}_n$.
This approach is similar in spirit to stopping criteria used for conjugate gradient approximations of GPs \citep{mackay1997efficient,artemev2021tighter}.
We start with the following lower bound on $\sigma^{2(\text{L})}_n$, proved in \cref{app:lb}.
\begin{proposition}\label{prop:lb}
    Let $\sigma^{2(\text{L})}_n$ be as in \cref{eq:pv}. Since $\tilde{\mathbf{S}}$ is diagonal and PD, wlog, we can rewrite it as $\tilde{\mathbf{S}} = \tilde{\mathbf{S}}' + \sigma^2 \mathbf{I}$, where $\tilde{\mathbf{S}}'$ is diagonal and PD. Then,
    \begin{align}
        \sigma^{2(\text{L})}_n &\geq k_{nn} - \frac{1}{\sigma^2} \left(\mathbf{k}_{n \mathbf{u}} \mathbf{k}_{\mathbf{u} n} - 2 \mathbf{k}_{n \mathbf{u}} \mathbf{T} \tilde{\mathbf{K}}_{-} \mathbf{k}_{\mathbf{u} n} \right. \nonumber \\
        &\qquad\left.+ \mathbf{k}_{n \mathbf{u}} \mathbf{T} \tilde{\mathbf{K}}_{-} \tilde{\mathbf{K}} \mathbf{T}\mathbf{k}_{\mathbf{u} n} \right) =: L_n, \label{eq:lsvgp_lb}
    \end{align}
    where $\tilde{\mathbf{K}}_{-} = \mathbf{K}_{\mathbf{u}\mathbf{u}} + \tilde{\mathbf{S}}'$, so that $\tilde{\mathbf{K}} = \tilde{\mathbf{K}}_{-} + \sigma^2 \mathbf{I}$.
\end{proposition}
By subtracting the upper and lower bounds on $\sigma^{2(\text{L})}_n$, we find the quantity
\begin{equation*}
    G_n := U_n - L_n = ||(\mathbf{I} - \tilde{\mathbf{K}} \mathbf{T}) \mathbf{k}_{\mathbf{u}n}||^2 / \sigma^2.
\end{equation*}
Since $\sigma_n^2$ enters the expected log-likelihood term of the ELBO as $-\sum_n \sigma_n^2 /(2\sigma^2_{\text{obs}})$, stopping when $G:=\sum_n G_n \le 2\sigma^2_{\text{obs}}\epsilon$ ensures the R-SVGP to L-SVGP gap from variance approximation is at most $\epsilon$.
While less general, this criterion measures distance in the ELBO's natural units and, like the first, is minimised at $\mathbf T=\tilde{\mathbf K}^{-1}$.
The per-datum computation of $G$ is quadratic in $M$, and suitable values of $\epsilon$ depend on the prediction task at hand.

\subsection{Other Practical Considerations}\label{sec:practical}
\paragraph{Bound Evaluation.}
While inverse-free, a direct evaluation of the R-SVGP bound \eqref{eq:rsvgp_precond} still entails $\mathcal{O}(B M^2 + M^3)$ due to the cubic trace terms in the KL bound, each involving products of three or more $M \times M$ matrices.
Luckily, cheap unbiased estimators of such traces can be obtained via Hutchinson's method \citep{hutchinson1989stochastic}, which only require matrix vector products, reducing the cost to $\mathcal{O}((B + K) M^2)$ with $K$ random probe vectors used.
When $K$ is too small, however, the resulting ELBO estimates become noisy, which may hinder optimisation.
\cref{app:hutchinson} studies this runtime--accuracy trade-off on two regression datasets and shows that moderate values of $K$ can yield significant speedups without materially affecting performance.

\paragraph{Bound Optimisation.}
Despite the savings on evaluation, optimisation of R-SVGP remains cubic in $M$ due to the full matmuls in the NG \eqref{eq:natgrad}.
On modern accelerators, matmuls are highly optimised operations that are more suited to low-precision arithmetic than matrix decompositions.
Therefore, with the right hardware, the proposed inverse-free approach may be faster than other SVGP variants when the number of NG updates $t^*$ is moderate.
Furthermore, we conjecture that randomized matmul techniques \citep{drineas2016randnla} could further reduce NG cost.

\paragraph{Step Size Scheduling.}
Given a stopping criterion, we seek a step-size schedule $\{\gamma_t\}$ that minimises the number of NG updates $t^*$ necessary to satisfy it.
We find that a suitable schedule crucially depends on whether the inducing location $\mathbf{Z}$ is optimised or not.
When $\mathbf{Z}$ is trained, a log-linear schedule \citep{salimbeni2018natural} from a small value, e.g., $\gamma_0 = 10^{-5}$, to $\gamma_T = 1$, where it stays constant, over $10$ steps is a robust, if a bit conservative, choice.
On the other hand, when $\mathbf{Z}$ is fixed, a single NG step with $\gamma = 1$ typically suffices ($t^* = 1$), thus rendering the stopping criterion unnecessary.
This suggests rapid movement in $\mathbf{Z}$ is the dominant source of $\tilde{\mathbf{K}}^{-1}$ drift that prevents cheap NG convergence.
When runtime is a priority, we recommend mild regularisation of the learning of $\mathbf{Z}$ by
\begin{enuminline}
    \item freezing $\mathbf{Z}$ for the first few training iterations,
    \item using a smaller learning rate for $\mathbf{Z}$, and
    \item raising Adam's $\beta_1$ parameter from the default $0.9$ to $0.99$ for $\mathbf{Z}$ only.
\end{enuminline}
When coupled with $k$-means\texttt{++} initialisation of $\mathbf{Z}$ \citep{arthur2007kmeans} on shallow GP models, these heuristics usually make $\mathbf{T}$ optimisation much cheaper, often requiring only a small number $t^* \leq 3$ of NG updates with fixed step size $\gamma_t = 1$.
Furthermore, in some cases, they also improve the model's overall training performance, possibly helping tame instability of the gradients wrt $\mathbf{Z}$.

\paragraph{Computational Summary.}
W-SVGP and L-SVGP both have per-iteration cost $\mathcal{O}(B M^2 + M^3)$.
For R-SVGP, Hutchinson trace estimation reduces the cost of bound \emph{evaluation} from $\mathcal{O}(B M^2 + M^3)$ to $\mathcal{O}((B + K) M^2)$, where $K$ is the number of probe vectors.
\textit{Optimising} R-SVGP, however, still requires $t^*$ NG updates of $\mathbf{T}$, which contribute a cubic $\mathcal{O}(t^* M^3)$ cost, so that one outer training iteration costs $\mathcal{O}((B + K) M^2 + t^* M^3)$.
Hence, potential speedups over W-SVGP and L-SVGP depend on: the gains from cheaper bound evaluation, the number $t^*$ of NG updates required, and the practical speed of matrix multiplications relative to matrix decompositions on the hardware used, despite their identical cubic asymptotic scaling.
In \cref{sec:efficiency}, we illustrate a setting where, using the heuristics discussed above, these factors favour R-SVGP, leading to faster training than the other variants.

\section{RESULTS}\label{sec:results}
This section investigates three main questions:
\begin{enumerate}\setlength\itemsep{0em}
    \item \textbf{Efficacy}: do the tools introduced in \cref{sec:improving} solve the optimisation issues of R-SVGP?
    \item \textbf{Efficiency}: can R-SVGP be made faster than standard baselines?
    \item \textbf{Generality}: can R-SVGP be used as a drop-in replacement for complex SVGP-based models? 
\end{enumerate}
To answer these questions, we perform experiments on toy datasets, UCI regression benchmarks, and complex models such as deep GPs \citep{salimbeni2017doubly} and convolutional GPs \citep{vdw2017convolutional}.
We compare R-SVGP against W-SVGP, as it is the default parameterisation in popular GP libraries \citep[e.g.,][]{matthews2017gpflow}, and L-SVGP, as it is the parameterisation to which R-SVGP reverts when $\mathbf{T} = \tilde{\mathbf{K}}^{-1}$.
Unless otherwise specified, we write R-SVGP and L-SVGP to refer to the versions with inducing mean preconditioning (\cref{sec:preconditioning}) and, for R-SVGP, NG updates for $\mathbf{T}$ (\cref{sec:natgrad}), and we employ the standard ARD RBF kernel \citep{williams2006gaussian} for all models.
All results are averaged over 5 random seeds, reported as mean $\pm$ standard error (unless noted).
The setup of all experiments below is detailed in \cref{sec:experimental_details}.

\subsection{Efficacy: Toy Datasets}\label{sec:efficacy}
To assess the effectiveness of the NG updates for $\mathbf{T}$ (\cref{sec:natgrad}) and the inducing mean preconditioning (\cref{sec:preconditioning}), we compare versions of the R-SVGP bound that take advantage of either or both of these tools.
In the following, we use suffixes \textsc{N} and \textsc{P} to indicate the presence of NG updates for $\mathbf{T}$ and inducing mean preconditioning, respectively.
In the context of efficacy, the most interesting baseline is the R-SVGP bound without preconditioning and where $\mathbf{T}$ is trained with Adam, as in \citet{vdwilk2022improved}.

The benefits of the tools proposed in \cref{sec:improving} are apparent even with toy datasets.
\cref{fig:snelson_banana_trace} shows single-run training loss traces for each bound optimised on \textsc{snelson}, a one-dimensional regression task, and \textsc{banana}, a two-dimensional binary classification task.
In both cases, the R-SVGP bounds with NG updates (i.e., R-SVGP (N) and R-SVGP (NP)) match the performance of the corresponding L-SVGP bounds (i.e., L-SVGP and L-SVGP (P), respectively).
Conversely, the bounds trained only with Adam suffer from either much slower convergence (R-SVGP) or worse results (R-SVGP (P)) than their NG counterparts.
Furthermore, preconditioning is crucial for the L-SVGP and R-SVGP bounds to match the performance of W-SVGP.
However, for R-SVGP it helps \textit{only} when coupled with NG: the inverse-free preconditioner $\mathbf P^{(\mathrm R)}$ is an effective surrogate for $\tilde{\mathbf K}^{-1}$ only when $\mathbf T$ tracks $\tilde{\mathbf K}^{-1}$ closely.

Overall, R-SVGP (NP) is the only inverse-free variational bound that matches the performance of the Cholesky-based variants on both datasets.

\subsection{Efficiency: UCI Benchmarks}\label{sec:efficiency}
Next, we investigate whether a well-tuned version of R-SVGP, which takes advantage of the heuristics in \cref{sec:stopping,sec:practical}, can be made faster than standard baselines on realistic datasets.
To this end, we consider the \textsc{elevators} ($N = 16599$, $D = 18$) and \textsc{kin40k} ($N = 40000$, $D = 8$) UCI scalar regression datasets.
These datasets have different characteristics: SVGPs are known to perform well on \textsc{elevators}, with performance saturating at $M = 1000$ to $2000$ inducing points, while \textsc{kin40k} is a more challenging dataset where performance keeps improving as $M$ increases.

For all methods, given a number $M$ of inducing points, we initialise the inducing location $\mathbf{Z}$ with $k$-means\texttt{++} and train it alongside the other model parameters using Adam with a learning rate of $5 \times 10^{-3}$.
For R-SVGP, we apply the heuristics discussed in \cref{sec:practical} to regularise the optimisation of $\mathbf{Z}$, with details given in \cref{sec:experimental_details}.
In turn, this allows us to choose a fixed NG step size $\gamma_t = 1$ for $\mathbf{T}$ optimisation, which we stop adaptively using the general criterion \eqref{eq:stopping_frobenius} with tolerance $\epsilon = 5 \times 10^{-3}$.
All methods are trained for $20000$ iterations with batch size $B = 100$.
\cref{fig:uci_nlpd} plots the test negative log-predictive density (NLPD) against training time for each method on both datasets, with $M$ ranging from $1000$ to $4000$.
\begin{figure*}
    \centering
    \includegraphics[width=1.0\textwidth]{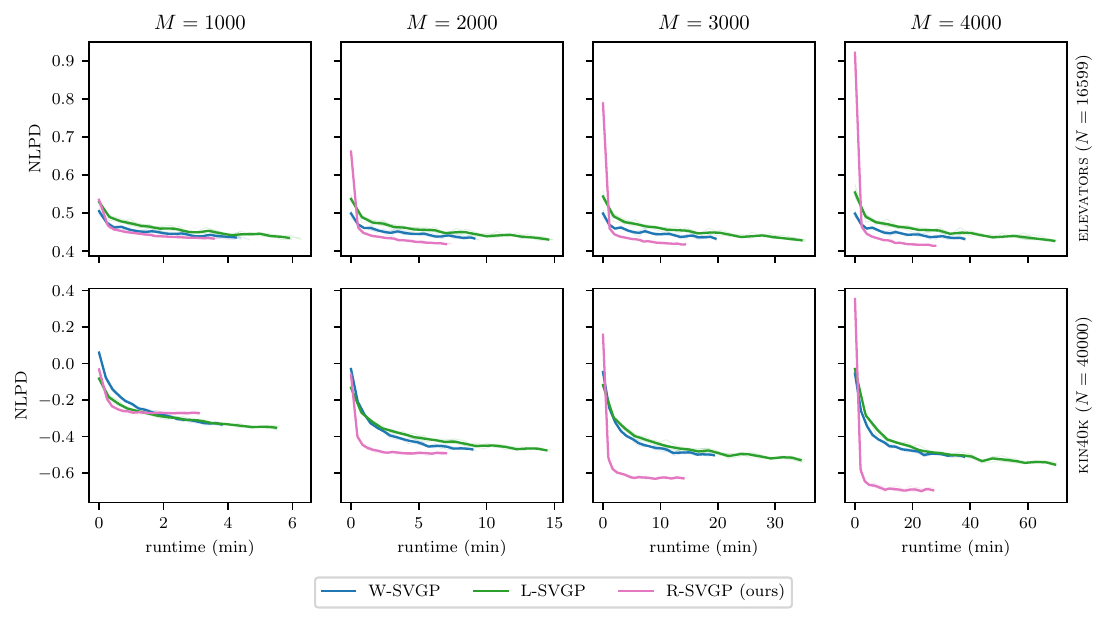}
    \caption{NLPD/runtime on \textsc{elevators} and \textsc{kin40k} datasets for different choices of $M$ and batch size $B = 100$. Lines show the mean over 5 seeds; shaded regions indicate $\pm 1$ standard error (often smaller than line width).}
    \label{fig:uci_nlpd}
\end{figure*}
As expected, the performance of all methods is stable for $M \geq 2000$ on \textsc{elevators}, while it improves as $M$ increases on \textsc{kin40k}.
Overall, R-SVGP shows faster convergence than both W-SVGP and L-SVGP, with a total training time that is up to $2.5\times$ faster than L-SVGP and up to $1.4\times$ faster than W-SVGP in our setup.
On \textsc{elevators}, R-SVGP achieves comparable NLPD to W-SVGP and L-SVGP, while, on \textsc{kin40k}, R-SVGP attains significantly better performance than other baselines for all $M$ except for $M = 1000$.
This shows a potential performance compromise when applying the heuristics from \cref{sec:practical} to the optimisation of $\mathbf{Z}$: while they can improve the overall training dynamics when $M$ is sufficiently large (e.g., $M \geq 2000$ on \textsc{kin40k}), regularising $\mathbf{Z}$ updates too much may lead to suboptimal performance when $M$ is much smaller than the optimal number for a given dataset (e.g., $M = 1000$ on \textsc{kin40k}). 

\subsection{Generality: Beyond Scalar Shallow SVGP}\label{sec:generality}
So far our experiments used shallow, scalar-output SVGPs with RBF kernels.
Here we test whether the R-SVGP bound can be used for more complex SVGP-based models, including multi-output GPs \citep{vdw2020framework}, deep GPs \citep{salimbeni2017doubly}, and non-standard kernels such as convolutional kernels \citep{vdw2017convolutional}.
Our goal in this subsection is \emph{efficacy}, not efficiency: we ask whether R-SVGP can be successfully optimised and compete with standard baselines in terms of predictive performance, without optimising for the fastest training setup.
Accordingly, we \textit{do not} apply the heuristics of \cref{sec:practical} to $\mathbf{Z}$.
For $\mathbf{T}$, we run natural gradients (NG) with a conservative increasing log-linear step-size schedule from $\gamma_0=10^{-5}$ to $\gamma_T=1$ over $10$ steps, and select $t^*$ adaptively using the general criterion \eqref{eq:stopping_frobenius} with a strict tolerance $\epsilon$ between $10^{-9}$ and $10^{-6}$ depending on the specific experiment.
This intentionally conservative routine is \emph{not intended} to yield a speed-up over baselines; our aim is to verify that R-SVGP achieves comparable predictive performance in multi-output, deep, and convolutional-kernel settings.
While we expect the heuristics of \cref{sec:practical} to transfer to these settings, doing so is non-trivial. In particular, regularising the updates of $\mathbf Z$ presumes a good initialisation of the latter: for shallow GPs, $k$-means\texttt{++} is effective, but in more complex models finding a suitable initialisation is more challenging, and we leave a systematic study to future work.

\paragraph{Deep GPs.}
In the formulation of \citet{salimbeni2017doubly}, a deep GP (DGP) is built by stacking $L$ SVGP layers: each layer $\ell \in \{1,\dots,L\}$ has its own inducing location $\mathbf{Z}^{(\ell)}$ and inducing variable $\mathbf{u}^{(\ell)}$, the variational posterior factorises as $\prod_{\ell=1}^L q(\mathbf u^{(\ell)})$, and the ELBO equals a Monte Carlo estimate of the expected log-likelihood through the latent layers minus the sum of layer-wise KL terms.
Consequently, R-SVGP applies to the deep setting layer-wise, with each layer endowed with its own $\mathbf{T}^{(\ell)}$ optimised independently using the tools of \cref{sec:improving}.
When modelling a scalar function $f \colon \mathbb{R}^D\!\to\!\mathbb{R}$, standard DGP implementations \citep{dutordoir2021gpflux} use multi-output hidden layers with width $Q_\ell$ (often $Q_\ell=D$): each hidden layer is modelled as $Q_\ell$ independent SVGPs sharing $\mathbf{Z}^{(\ell)}$ and kernel hyperparameters, but with per-output variational parameters $\{\tilde{\mathbf{m}}^{(\ell,q)}, \tilde{\mathbf{S}}^{(\ell,q)}\}_{q=1}^{Q_\ell}$.
For W-SVGP layers, this sharing implies a single Cholesky decomposition of $\mathbf{K}_{\mathbf{u}^{(\mathbf{\ell})}\mathbf{u}^{(\ell)}}$ is required per hidden layer, regardless of $Q_\ell$.
For L-SVGP or R-SVGP layers, the matrix of interest becomes $\tilde{\mathbf{K}}^{(\ell,q)} = \mathbf{K}_{\mathbf{u}^{(\ell)}\mathbf{u}^{(\ell)}} + \tilde{\mathbf{S}}^{(\ell,q)}$, which depends on the output via $\tilde{\mathbf{S}}^{(\ell,q)}$.
Although each $\tilde{\mathbf K}^{(\ell,q)}$ remains $M\times M$, one now needs $Q_\ell$ distinct factorisations (for L-SVGP) or $Q_\ell$ NG inner loops (for R-SVGP) per training step.
To avoid this mismatch, we share a single $\tilde{\mathbf S}^{(\ell)}$ across outputs within a layer, which restores one decomposition/NG inner loop at the cost of reduced variational flexibility.

We evaluate the three DGP parameterisations on \textsc{kin40k}, ranging $L$ from $1$ to $3$ layers and using $M = 500$ inducing points per layer.
\cref{tab:dgp_kin40k} reports the test NLPD for each method.
\begin{table}
    \caption{NLPD on \textsc{kin40k} for different choices of $L$, $M\!=\!500$ and batch size $B\!=\! 1000$, after 200k iterations. Entries are mean (standard error) over 5 seeds.}
    \label{tab:dgp_kin40k}
    \centering
    \begin{tabular}{lrrr}
        \toprule
        $L$ & W-SVGP & L-SVGP & R-SVGP \\
        \midrule
        1 & -0.38 (0.001) & -0.39 (0.002) & -0.39 (0.002) \\
        2 & -1.86 (0.03) & -1.88 (0.02) & -1.88 (0.02) \\
        3 & -2.39 (0.03) & -2.64 (0.01) & -2.64 (0.01) \\
        \bottomrule
    \end{tabular}
\end{table}
All methods improve as $L$ increases, showcasing the benefits of depth on this dataset. R-SVGP matches the performance of L-SVGP for all $L$, proving that the inverse-free bound can be successfully optimised in deep models using the tools of \cref{sec:improving}.
W-SVGP performs slightly worse for $L = 3$, perhaps because of the increased difficulty of optimising the more flexible variational posterior in deeper models.

\paragraph{Convolutional GPs.}
Convolutional GPs (ConvGPs) \citep{vdw2017convolutional} encode translation–equivariant structure in SVGPs via specialised kernels and interdomain inducing variables that live in patch space rather than input space.
Consequently, R-SVGP and L-SVGP apply unchanged: one simply substitutes the interdomain covariances $(\mathbf{K}_{\mathbf{u}\mathbf{u}},\mathbf{k}_{\mathbf{u}n})$ in the standard formulas.
For multi-class classification with $C$ classes, ConvGPs are typically implemented as $C$ independent per-class latent GPs that share inducing locations $\mathbf{Z}$ and kernel hyperparameters.
Hence, the same cost considerations as in multi-output SVGPs apply: to avoid $C$ per-step factorisations (L-SVGP) or NG inner loops (R-SVGP), we tie a single-channel $\tilde{\mathbf S}$ across classes, trading some variational flexibility relative to W-SVGP for matched computational cost.

We follow the same experimental setup as in \citet{vdw2017convolutional} to evaluate the three parameterisations on the \textsc{mnist} dataset.
\cref{tab:mnist} compares the test error of each method using either a standard RBF kernel or a weighted convolutional kernel \citep{vdw2017convolutional} at the end of training.
\begin{table}
    \caption{Test error (\%) on \textsc{mnist} with RBF and weighted convolutional kernels, $M=750$, $B=200$. Entries are mean (standard error) over 5 seeds.}
    \label{tab:mnist}
    \centering
    \begin{tabular}{lrrr}
        \toprule
        Kernel & W-SVGP & L-SVGP & R-SVGP \\
        \midrule
        RBF & 1.96 (0.02) & 1.95 (0.01) & 1.94 (0.02) \\
        Conv & 1.31 (0.02) & 1.53 (0.02) & 1.53 (0.03) \\
        \bottomrule
    \end{tabular}
\end{table}
All methods benefit from the inductive bias of the convolutional kernel, with W-SVGP slightly outperforming the other methods because of its more flexible variational posterior.
R-SVGP matches the performance of L-SVGP, confirming that the effectiveness of the tools in \cref{sec:improving} extends to non-standard kernels.

\section{RELATED WORK}\label{sec:related}
Every Gaussian process method relies on a numerical (approximation) method for calculating inverse-vector products and determinants, with the Cholesky decomposition being commonly used because of its accuracy.
However, due to its $O(N^2)$ memory and $O(N^3)$ time costs, alternatives have long been investigated. 

In particular, iterative Krylov-subspace solvers \citep{liesen2013krylov} and Conjugate Gradients (CG) \citep{hestenes1952methods} were proposed as alternatives \citep{mackay1997efficient,davies2015effective}.
When terminated early, these methods can provide good approximations at lower computational and memory ($O(N)$) cost, and their reliance on matmuls make them well-suited to modern GPU hardware \citep{gardner2018gpytorch,wang2019exact}. However, selecting when to terminate drastically affects the stability and quality of the results, although partial solutions for automatically finding a good trade-off exist \citep{artemev2021tighter}.

The methods above are limited to full-batch regression with Gaussian likelihoods, where the marginal likelihood is a log Gaussian density evaluation.
By contrast, variational approximations \citep{titsias2009variational} support non-Gaussian likelihoods \citep{hensman2015mcmc}, mini-batching \citet{hensman2013gaussian}, and deep structures \citep{salimbeni2017doubly}.
In these settings, CG is not applicable in the same way, and existing methods therefore rely on Cholesky factorisations, which are poorly matched to low-precision, parallel hardware.

We attempt to bridge these lines of work by proposing an objective for GP inference that is hardware-friendly yet retains the benefits of the SVGP framework.

\section{DISCUSSION}\label{sec:discussion}
This work introduces, to our knowledge, the first practical inverse-free sparse variational GP method that trains stably on realistic datasets without matrix decompositions.
Using only matrix multiplications (matmuls), R-SVGP attains comparable predictive performance to Cholesky-based baselines and serves as a drop-in replacement in deeper and interdomain SVGP models, while offering potential speed-ups when well tuned.
Several avenues for future work remain.

While inverse-free, R-SVGP remains cubic in $M$ due to the matmuls required by the NG optimisation of $\mathbf T$.
Consequently, potential speed-ups over Cholesky-based methods depend on keeping the NG inner loop short.
We provide heuristics that help in shallow GPs, but their benefit in deep, multi-output, or interdomain settings merits further study and tuning for speed.

Moreover, R-SVGP is built on L-SVGP, which, in multi-output regimes, trades some variational flexibility to match the computational cost of W-SVGP.
This suggests that alternative inverse-free parameterisations targeting W-SVGP, such as the construction sketched in \cref{sec:svgp}, represent a promising direction.

Finally, lower-precision regimes (e.g., FP16/FP32), where matmuls fully exploit modern accelerators, may amplify the potential speed and memory advantages of R-SVGP over Cholesky-based alternatives, and we are excited to investigate this systematically.

\subsection*{Acknowledgements}

The authors would like to acknowledge the use of the University of Oxford Advanced Research Computing (ARC) facility \citep{richards2015} in carrying out this work.
Stefano Cortinovis is supported by the EPSRC Centre for Doctoral Training in Modern Statistics and Statistical Machine Learning (EP/S023151/1).

\bibliography{references}

@article{abadi2016tensorflow,
  title={Tensor{F}low: Large-scale machine learning on heterogeneous distributed systems},
  author={Abadi, Mart{\'\i}n and Agarwal, Ashish and Barham, Paul and Brevdo, Eugene and Chen, Zhifeng and Citro, Craig and Corrado, Greg S. and Davis, Andy and Dean, Jeffrey and Devin, Matthieu and others},
  journal={arXiv preprint arXiv:1603.04467},
  year={2016}
}

@article{amari1998natural,
  title={Natural gradient works efficiently in learning},
  author={Amari, Shun-Ichi},
  journal={Neural computation},
  volume={10},
  number={2},
  pages={251--276},
  year={1998},
  publisher={MIT Press}
}

@inproceedings{artemev2021tighter,
  title={Tighter bounds on the log marginal likelihood of {G}aussian process regression using conjugate gradients},
  author={Artemev, Artem and Burt, David R. and van der Wilk, Mark},
  booktitle={International Conference on Machine Learning},
  pages={362--372},
  year={2021},
  organization={PMLR}
}

@inproceedings{arthur2007kmeans,
  author = {Arthur, David and Vassilvitskii, Sergei},
  title = {k-means++: the advantages of careful seeding},
  year = {2007},
  isbn = {9780898716245},
  publisher = {Society for Industrial and Applied Mathematics},
  address = {USA},
  booktitle = {Proceedings of the Eighteenth Annual ACM-SIAM Symposium on Discrete Algorithms},
  pages = {1027–1035},
  numpages = {9},
  location = {New Orleans, Louisiana},
  series = {SODA '07}
}

@article{ben1965iterative,
  title={An iterative method for computing the generalized inverse of an arbitrary matrix},
  author={Ben-Israel, Adi},
  journal={Mathematics of Computation},
  pages={452--455},
  year={1965},
  publisher={JSTOR}
}

@article{blei2017variational,
  title={Variational inference: A review for statisticians},
  author={Blei, David M. and Kucukelbir, Alp and McAuliffe, Jon D.},
  journal={Journal of the American statistical Association},
  volume={112},
  number={518},
  pages={859--877},
  year={2017},
  publisher={Taylor \& Francis}
}

@inproceedings{burt2019rates,
  title = 	 {Rates of Convergence for Sparse Variational {G}aussian Process Regression},
  author =       {Burt, David and Rasmussen, Carl Edward and van der Wilk, Mark},
  booktitle = 	 {Proceedings of the 36th International Conference on Machine Learning},
  pages = 	 {862--871},
  year = 	 {2019},
  editor = 	 {Chaudhuri, Kamalika and Salakhutdinov, Ruslan},
  volume = 	 {97},
  series = 	 {Proceedings of Machine Learning Research},
  month = 	 {09--15 Jun},
  publisher =    {PMLR},
  url = 	 {https://proceedings.mlr.press/v97/burt19a.html}
}

@article{burt2020convergence,
  author  = {David R. Burt and Carl Edward Rasmussen and Mark van der Wilk},
  title   = {Convergence of Sparse Variational Inference in {G}aussian Processes Regression},
  journal = {Journal of Machine Learning Research},
  year    = {2020},
  volume  = {21},
  number  = {131},
  pages   = {1--63},
  url     = {http://jmlr.org/papers/v21/19-1015.html}
}

@phdthesis{davies2015effective,
  title        = {Effective implementation of {G}aussian process regression for machine learning},
  author       = {Alexander Davies},
  year         = 2015,
  school       = {University of Cambridge},
  type         = {PhD thesis}
}

@article{drineas2016randnla,
  title={RandNLA: randomized numerical linear algebra},
  author={Drineas, Petros and Mahoney, Michael W},
  journal={Communications of the ACM},
  volume={59},
  number={6},
  pages={80--90},
  year={2016},
  publisher={ACM New York, NY, USA}
}

@misc{dutordoir2021gpflux,
      title={{GP}flux: A Library for Deep {G}aussian Processes}, 
      author={Vincent Dutordoir and Hugh Salimbeni and Eric Hambro and John McLeod and Felix Leibfried and Artem Artemev and Mark van der Wilk and James Hensman and Marc P. Deisenroth and ST John},
      year={2021},
      eprint={2104.05674},
      archivePrefix={arXiv},
      primaryClass={stat.ML},
      url={https://arxiv.org/abs/2104.05674}, 
}

@article{gardner2018gpytorch,
  title={{GP}y{T}orch: Blackbox matrix-matrix {G}aussian process inference with {GPU} acceleration},
  author={Gardner, Jacob and Pleiss, Geoff and Weinberger, Kilian Q and Bindel, David and Wilson, Andrew G},
  journal={Advances in neural information processing systems},
  volume={31},
  year={2018}
}

@article{hensman2013gaussian,
  title={Gaussian processes for big data},
  author={Hensman, James and Fusi, Nicolo and Lawrence, Neil D.},
  journal={arXiv preprint arXiv:1309.6835},
  year={2013}
}

@article{hensman2015mcmc,
  title={{MCMC} for variationally sparse {G}aussian processes},
  author={Hensman, James and Matthews, Alexander G. de G. and Filippone, Maurizio and Ghahramani, Zoubin},
  journal={Advances in neural information processing systems},
  volume={28},
  year={2015}
}

@article{hestenes1952methods,
  title={Methods of conjugate gradients for solving linear systems},
  author={Hestenes, Magnus R and Stiefel, Eduard and others},
  journal={Journal of research of the National Bureau of Standards},
  volume={49},
  number={6},
  pages={409--436},
  year={1952}
}

@article{hutchinson1989stochastic,
  title={A stochastic estimator of the trace of the influence matrix for {L}aplacian smoothing splines},
  author={Hutchinson, Michael F},
  journal={Communications in Statistics-Simulation and Computation},
  volume={18},
  number={3},
  pages={1059--1076},
  year={1989},
  publisher={Taylor \& Francis}
}

@book{kelley1995iterative,
  title={Iterative methods for linear and nonlinear equations},
  author={Kelley, Carl T},
  year={1995},
  publisher={SIAM}
}

@article{kingma2014adam,
  title={Adam: A method for stochastic optimization},
  author={Kingma, Diederik P.},
  journal={arXiv preprint arXiv:1412.6980},
  year={2014}
}

@book{liesen2013krylov,
  author    = {Liesen, Jörg and Strakoš, Zdeněk},
  title     = {Krylov Subspace Methods: Principles and Analysis},
  publisher = {Oxford University Press},
  year      = {2013},
  address   = {Oxford},
  isbn      = {978-0-19-965541-0}
}

@article{mackay1997efficient,
  title={Efficient implementation of {G}aussian processes},
  author={MacKay, David J. C. and Gibbs, Mark},
  journal={Neural Computation},
  year={1997}
}

@article{matthews2017gpflow,
  title={{GPflow}: A {G}aussian process library using {T}ensor{F}low},
  author={Matthews, Alexander G. de G. and van der Wilk, Mark and Nickson, Tom and Fujii, Keisuke and Boukouvalas, Alexis and Le, Pablo and Ghahramani, Zoubin and Hensman, James and others},
  journal={Journal of Machine Learning Research},
  volume={18},
  number={40},
  pages={1--6},
  year={2017}
}

@phdthesis{matthewsthesis,
  title={Scalable {G}aussian process inference using variational methods},
  author={Matthews, Alexander G. de G.},
  school={University of Cambridge},
  year={2017}
}

@misc{panosFullyScalableGaussian2018,
	title = {Fully scalable {G}aussian processes using subspace inducing inputs},
	url = {http://arxiv.org/abs/1807.02537},
	doi = {10.48550/arXiv.1807.02537},
	urldate = {2024-03-15},
	publisher = {arXiv},
	author = {Panos, Aristeidis and Dellaportas, Petros and Titsias, Michalis K.},
	month = jul,
	year = {2018},
	note = {arXiv:1807.02537 [cs, stat]},
}

@article{quinonero2005unifying,
  title={A unifying view of sparse approximate {G}aussian process regression},
  author={Qui{\~n}onero-Candela, Joaquin and Rasmussen, Carl E.},
  journal={The Journal of Machine Learning Research},
  year={2005},
}

@misc{richards2015,
  title={University of {O}xford {A}dvanced {R}esearch {C}omputing},
  author={Richards, Andrew},
  year={2015},
  doi={Zenodo.10.5281/zenodo.22558}
}

@article{salimbeni2017doubly,
  title={Doubly stochastic variational inference for deep {G}aussian processes},
  author={Salimbeni, Hugh and Deisenroth, Marc},
  journal={Advances in neural information processing systems},
  volume={30},
  year={2017}
}

@inproceedings{salimbeni2018natural,
  title={Natural gradients in practice: Non-conjugate variational inference in {G}aussian process models},
  author={Salimbeni, Hugh and Eleftheriadis, Stefanos and Hensman, James},
  booktitle={International Conference on Artificial Intelligence and Statistics},
  pages={689--697},
  year={2018},
  organization={PMLR}
}

@article{snelson2005sparse,
  title={Sparse {G}aussian processes using pseudo-inputs},
  author={Snelson, Edward and Ghahramani, Zoubin},
  journal={Advances in neural information processing systems},
  volume={18},
  year={2005}
}

@inproceedings{titsias2009variational,
  title = 	 {Variational Learning of Inducing Variables in Sparse {G}aussian Processes},
  author = 	 {Titsias, Michalis},
  booktitle = 	 {International Conference on Artificial Intelligence and Statistics},
  pages = 	 {567--574},
  year = 	 {2009},
  organization={PMLR}
}

@article{vdw2017convolutional,
  title={Convolutional {G}aussian processes},
  author={van der Wilk, Mark and Rasmussen, Carl Edward and Hensman, James},
  journal={Advances in neural information processing systems},
  volume={30},
  year={2017}
}

@article{vdw2020framework,
  title={A framework for interdomain and multioutput {G}aussian processes},
  author={van der Wilk, Mark and Dutordoir, Vincent and John, ST and Artemev, Artem and Adam, Vincent and Hensman, James},
  journal={arXiv preprint arXiv:2003.01115},
  year={2020}
}

@inproceedings{vdwilk2020variational,
  title={Variational {G}aussian process models without matrix inverses},
  author={van der Wilk, Mark and John, ST and Artemev, Artem and Hensman, James},
  booktitle={Symposium on Advances in Approximate Bayesian Inference},
  pages={1--9},
  year={2020},
  organization={PMLR}
}

@inproceedings{vdwilk2022improved,
  title={Improved Inverse-Free Variational Bounds for Sparse {G}aussian Processes},
  author={van der Wilk, Mark and Artemev, Artem and Hensman, James},
  booktitle={Fourth Symposium on Advances in Approximate Bayesian Inference},
  year={2022}
}

@phdthesis{vdwthesis,
  title={Sparse {G}aussian process approximations and applications},
  author={van der Wilk, Mark},
  year={2019},
  school={University of Cambridge},
  url={https://www.repository.cam.ac.uk/handle/1810/288347},
}

@article{wang2019exact,
  title={Exact {G}aussian processes on a million data points},
  author={Wang, Ke and Pleiss, Geoff and Gardner, Jacob and Tyree, Stephen and Weinberger, Kilian Q and Wilson, Andrew Gordon},
  journal={Advances in neural information processing systems},
  volume={32},
  year={2019}
}

@book{williams2006gaussian,
  title={Gaussian processes for machine learning},
  author={Williams, Christopher K. I. and Rasmussen, Carl E.},
  volume={2},
  year={2006},
  publisher={MIT press Cambridge, MA}
}

\clearpage

\appendix
\onecolumn

\begin{center}
    \Large \textbf{Inverse-Free Sparse Variational Gaussian Processes: \\
Supplementary Material}
\end{center}

\setcounter{section}{0}
\setcounter{table}{0}
\setcounter{figure}{0}
\setcounter{equation}{0}

\makeatletter
\renewcommand \thesection{S\@arabic\c@section}
\renewcommand \thetable{S\@arabic\c@table}
\renewcommand \thefigure{S\@arabic\c@figure}
\renewcommand \theequation{S\arabic{equation}}
\makeatother

The supplementary material contains the proofs of \cref{prop:natgrad,prop:grad,prop:lb} in \cref{sec:improving}, a brief derivation of the alternative inverse-free bound mentioned in \cref{sec:svgp}, as well as additional details on the experiments in \cref{sec:results} and further experimental results.
All sections and equations in the supplementary material are prefixed with an `S' for clarity.

\section{PROOF OF PROPOSITION \texorpdfstring{\ref{prop:natgrad}}{1}}\label{app:natgrad}
\begin{proof}
    Given the decomposition for our variational parameter $\mathbf{T} = \mathbf{L}\mathbf{L}^\top\!$, our goal is to find the optimal $\mathbf{L}$ so that $\mathbf{T} = (\mathbf{K_{uu}} + \tilde{\mathbf{S}})^{-1}$, without explicitly computing the inverse.
    We consider the problem of finding the Cholesky factor of the inverse of a matrix $\mathbf A$ via:
    \begin{align}
        \mathbf L_{\mathbf A^{-1}} = \arg\min_{\mathbf L}\,\ell_{\mathbf A}(\mathbf L) \qquad \ell_{\mathbf A}(\mathbf L) := \KL\left[\mathcal{N}(\mathbf 0, \mathbf L \mathbf L^\top)\vert\vert\mathcal{N}(\mathbf 0, \mathbf A^{-1})\right].
    \end{align}
    The above optimisation problem describes the update we take in the direction of the natural gradient:
    \begin{equation}
        \tilde{\mathbf g} = \mathbf{F}^{-1}\mathbf{g}, \label{eq:natgrad_formal}
    \end{equation}
    where $\mathbf g$ and $\tilde{\mathbf g}$ are column vectors denoting the gradient and the natural gradient of $\ell_\mathbf{A}$ wrt $\mathbf L$, while we denote with $\mathbf{F}$ the Fisher information matrix of the distribution $p(\mathbf{x}; \mathbf{L}) = \mathcal{N}(\mathbf x \vert\mathbf{0}, \mathbf{L}\mathbf{L}^\top)$ parameterised by $\mathbf{L}$:
    \begin{equation}
        \mathbf{F} = -\mathbb{E}_{p}\left[\frac{\partial^2}{\partial \vect(\mathbf{L})\ \partial \vect(\mathbf{L})^\top} \log p(\mathbf{x}; \mathbf{L})\right] \label{eq:fisher},
    \end{equation}
    with $\vect(\cdot)$ the column-wise vectorisation operator.
    
    We start by computing the gradient of the loss $\nabla \ell_\mathbf{A}$ wrt $\mathbf{L}$. As in \cref{eq:natgrad_loss},
    \begin{equation}
        \ell_\mathbf{A}(\mathbf{L}) = \KL\left[p(\mathbf{x}; \mathbf{L}) || \mathcal{N}(\mathbf{0}, \mathbf{A}^{-1})\right] = \frac{1}{2} \left(\tr(\mathbf{A} \mathbf{L} \mathbf{L}^\top) -\log |\mathbf{L}\mathbf{L}^\top| \right) + c_1,
    \end{equation}
    where $c_1$ is a constant that does not depend on $\mathbf{L}$. The first differential of $\ell_\mathbf{A}$ with respect to $\mathbf{L}$ is given by
    \begin{equation}
        \diff \ell_\mathbf{A} = \frac{1}{2} \underbrace{\diff \tr (\mathbf{A} \mathbf{L}\mathbf{L}^\top)}_{(a)} - \frac{1}{2} \underbrace{\diff \log |\mathbf{L}\mathbf{L}^\top|}_{(b)}. \label{eq:diff_ellA}
    \end{equation}
    Term $(a)$ can be computed as
    \begin{align}
        \diff \tr (\mathbf{A} \mathbf{L}\mathbf{L}^\top) &= \tr(\mathbf{A} \diff (\mathbf{L} \mathbf{L}^\top)) \nonumber \\
        &= \tr(\mathbf{A} \diff \mathbf{L} \mathbf{L}^\top + \mathbf{A} \mathbf{L} \diff \mathbf{L}^\top) \nonumber \\
        &= \tr(\mathbf{L}^\top \mathbf{A} \diff \mathbf{L}) + \tr(\mathbf{A} \mathbf{L} \diff \mathbf{L}^\top) \nonumber \\
        &= 2\vect(\mathbf A \mathbf L)^\top \vect(\diff \mathbf L),
    \end{align}
    where the last equality holds because $\mathbf{A}$ is symmetric. Term $(b)$ is given by
    \begin{align}
        \diff \log |\mathbf{L}\mathbf{L}^\top| &= \frac{1}{|\mathbf{L}\mathbf{L}^\top|} \diff |\mathbf{L}\mathbf{L}^\top| \nonumber \\
        &= \frac{1}{|\mathbf{L}\mathbf{L}^\top|} |\mathbf{L}\mathbf{L}^\top| \tr(\mathbf{L}^{-\top} \mathbf{L}^{-1} \diff (\mathbf{L} \mathbf{L}^\top)) \nonumber \\
        &= \tr(\mathbf{L}^{-\top} \mathbf{L}^{-1} \diff \mathbf{L} \mathbf{L}^\top + \mathbf{L}^{-\top} \mathbf{L}^{-1}  \mathbf{L} \diff \mathbf{L}^\top) \nonumber \\
        &= \tr(\mathbf{L}^{-1} \diff \mathbf{L}) + \tr(\mathbf{L}^{-\top} \diff \mathbf{L}^\top) \nonumber \\
        &= 2\vect(\mathbf L^{-\top})^\top \vect(\diff \mathbf L) \label{eq:diff_logLLt}.
    \end{align}
    Putting it all together, the differential of $\ell_\mathbf{A}$ takes the form
    \begin{equation}
        \diff \ell_\mathbf{A} = (\vect(\mathbf A \mathbf L) - \vect(\mathbf L^{-\top}))^\top \vect(\diff \mathbf L),
    \end{equation}
    which implies that
    \begin{equation}
        \mathbf g = \vect(\mathbf A \mathbf{L}) - \vect(\mathbf{L}^{-\top}).
    \end{equation}

    We now move to computing the Fisher information matrix $\mathbf{F}$ of $p(\mathbf{x}; \mathbf{L})$ wrt $\mathbf{L}$. We have that
    \begin{equation}
        \log p(\mathbf{x}; \mathbf{L}) = -\frac{1}{2} \log |\mathbf{L}\mathbf{L}^\top| - \frac{1}{2} \mathbf{x}^\top \mathbf{L}^{-\top}\mathbf{L}^{-1} \mathbf{x} + c_2,
    \end{equation}
    where $c_2$ is a constant that does not depend on $\mathbf{L}$. The quantity of interest is the expected value of the second differential of $\log p(\mathbf{x}; \mathbf{L})$ wrt $p$, which takes the form
    \begin{equation}
        \mathbb{E}_p \left[\diff^2 \log p(\mathbf{x}; \mathbf{L})\right] = -\frac{1}{2} \underbrace{\diff^2 \log |\mathbf{L}\mathbf{L}^\top|}_{(c)} - \frac{1}{2} \mathbb{E}_p \left[\underbrace{\diff^2 (\mathbf{x}^\top \mathbf{L}^{-\top}\mathbf{L}^{-1} \mathbf{x})}_{(d)}\right] \label{eq:diff}
    \end{equation}
    because $\log |\mathbf{L}\mathbf{L}^\top|$ does not depend on $\mathbf{x}$. By taking advantage of \cref{eq:diff_logLLt}, term $(c)$ is given by
    \begin{align}
        \diff^2 \log |\mathbf{L}\mathbf{L}^\top| &= 2 \diff \tr(\mathbf{L}^{-1} \diff \mathbf{L}) \nonumber \\
        &= -2 \tr(\mathbf{L}^{-1} \diff \mathbf{L} \mathbf{L}^{-1} \diff \mathbf{L}) \nonumber \\
        &= -2 \vect((\mathbf{L}^{-1} \diff \mathbf{L})^\top)^\top \vect(\mathbf{L}^{-1} \diff \mathbf{L}) \label{eq:d2_logLLt_step_2} \\
        &= -2 \left[\mathbf{C} \vect(\mathbf{L}^{-1} \diff \mathbf{L}))\right]^\top \vect(\mathbf{L}^{-1} \diff \mathbf{L}) \nonumber \\
        &= -2 \left[\mathbf{C}(\mathbf{I} \otimes \mathbf{L}^{-1}) \vect(\mathbf{d}\mathbf{L})\right]^\top (\mathbf{I} \otimes \mathbf{L}^{-1}) \vect(\diff\mathbf{L}) \nonumber \\
        &= -2 \vect(\mathbf{d}\mathbf{L})^\top (\mathbf{I} \otimes \mathbf{L}^{-\top}) \mathbf{C}(\mathbf{I} \otimes \mathbf{L}^{-1}) \vect(\diff \mathbf{L}) \nonumber \\
        &= -2 \vect(\diff\mathbf{L})^\top \mathbf{C} (\mathbf{L}^{-\top} \otimes \mathbf{L}^{-1}) \vect(\diff\mathbf{L}), \label{eq:d2_logLLt}
    \end{align}
    where \cref{eq:d2_logLLt_step_2} follows from $\tr(\mathbf{X}^\top \mathbf{Y}) = \vect(\mathbf{X})^\top \vect(\mathbf{Y})$, and $\mathbf{C} = \mathbf{C}^\top$ is the commutator matrix such that $\mathbf C \vect(\mathbf{X}) = \vect(\mathbf{X}^\top)$ and $\mathbf{C}(\mathbf{X} \otimes \mathbf{Y}) = (\mathbf{Y} \otimes \mathbf{X}) \mathbf{C}$.
    
    Similarly, term $(d)$ can be computed starting from the first differential
    \begin{align}
        \diff (\mathbf{x}^\top \mathbf{L}^{-\top}\mathbf{L}^{-1} \mathbf{x}) &= \mathbf{x}^\top \diff (\mathbf{L}^{-\top}\mathbf{L}^{-1}) \mathbf{x} \nonumber \\
        &= -\mathbf{x}^\top (\mathbf{L}^{-\top} \diff \mathbf{L}^\top \mathbf{L}^{-\top} \mathbf{L}^{-1} + \mathbf{L}^{-\top}\mathbf{L}^{-1} \diff \mathbf{L} \mathbf{L}^{-1}) \mathbf{x} \nonumber \\
        &= -\tr(\mathbf{L}^{-\top} \mathbf{L}^{-1}\mathbf{x} \mathbf{x}^\top \mathbf{L}^{-\top} \diff \mathbf{L}^\top) - \tr(\mathbf{L}^{-1} \mathbf{x} \mathbf{x}^\top \mathbf{L}^{-\top} \mathbf{L}^{-1} \diff \mathbf{L}) \nonumber \\
        &= -2\tr(\mathbf{L}^{-1} \mathbf{x} \mathbf{x}^\top \mathbf{L}^{-\top} \mathbf{L}^{-1} \diff \mathbf{L}).
    \end{align}
    Then, the second differential takes the form
    \begin{align}
        \diff^2 (\mathbf{x}^\top \mathbf{L}^{-\top}\mathbf{L}^{-1} \mathbf{x}) &= -2\tr(\diff(\mathbf{L}^{-1} \mathbf{x} \mathbf{x}^\top \mathbf{L}^{-\top} \mathbf{L}^{-1} \diff \mathbf{L})) \nonumber \\
        &= 2 \left[\tr(\mathbf{L}^{-1} \diff \mathbf{L} \mathbf{L}^{-1} \mathbf{x} \mathbf{x}^\top \mathbf{L}^{-\top} \mathbf{L}^{-1} \diff \mathbf{L}) \right. \nonumber \\ 
        &\qquad\qquad + \tr(\mathbf{L}^{-1} \mathbf{x} \mathbf{x}^\top \mathbf{L}^{-\top} \diff \mathbf{L}^\top \mathbf{L}^{-\top} \mathbf{L}^{-1} \diff \mathbf{L}) \nonumber \\
        &\qquad\qquad \left. + \tr(\mathbf{L}^{-1} \mathbf{x} \mathbf{x}^\top \mathbf{L}^{-\top} \mathbf{L}^{-1} \diff \mathbf{L} \mathbf{L}^{-1} \diff \mathbf{L}) \right].
    \end{align}
    By taking the expectation wrt $p$, we have that
    \begin{align}
        \mathbb{E}_p \left[\diff^2 (\mathbf{x}^\top \mathbf{L}^{-\top}\mathbf{L}^{-1} \mathbf{x}) \right] &= 2 \left[\tr(\mathbf{L}^{-1} \diff \mathbf{L} \mathbf{L}^{-1} \textcolor{blue}{\mathbf{L} \mathbf{L}^\top} \mathbf{L}^{-\top} \mathbf{L}^{-1} \diff \mathbf{L}) \right. \nonumber \\ 
        &\qquad\qquad + \tr(\mathbf{L}^{-1} \textcolor{blue}{\mathbf{L} \mathbf{L}^\top} \mathbf{L}^{-\top} \diff \mathbf{L}^\top \mathbf{L}^{-\top} \mathbf{L}^{-1} \diff \mathbf{L}) \nonumber \\
        &\qquad\qquad \left. + \tr(\mathbf{L}^{-1} \textcolor{blue}{\mathbf{L} \mathbf{L}^\top} \mathbf{L}^{-\top} \mathbf{L}^{-1} \diff \mathbf{L} \mathbf{L}^{-1} \diff \mathbf{L}) \right] \nonumber \\
        &= 2 \left[2 \tr(\mathbf{L}^{-1} \diff \mathbf{L} \mathbf{L}^{-1} \diff \mathbf{L}) + \tr(\diff \mathbf{L}^\top \mathbf{L}^{-\top} \mathbf{L}^{-1} \diff \mathbf{L}) \right] \nonumber \\
        &= 2 \left[2 \vect(\diff\mathbf{L})^\top \mathbf{C} (\mathbf{L}^{-\top} \otimes \mathbf{L}^{-1}) \vect(\diff\mathbf{L}) \right. \label{eq:d2_dxLLx_step_3} \\
        &\qquad\qquad \left.+ \vect(\diff \mathbf{L})^\top \vect(\mathbf{L}^{-\top} \mathbf{L}^{-1} \diff \mathbf{L})\right] \nonumber \\
        &= 2 \vect(\diff \mathbf{L})^\top \left[2 \mathbf{C} (\mathbf{L}^{-\top} \otimes \mathbf{L}^{-1}) + (\mathbf{I} \otimes \mathbf{L}^{-\top} \mathbf{L}^{-1})\right] \vect(\diff \mathbf{L}),
    \end{align}
    where the first component of \cref{eq:d2_dxLLx_step_3} follows from the expression already derived in \cref{eq:d2_logLLt}. Overall, the expectation in \cref{eq:diff} becomes
    \begin{equation}
        \mathbb{E}_p \left[\diff^2 \log p(\mathbf{x}; \mathbf{L})\right] = - \vect(\diff \mathbf{L})^\top \left[\mathbf{C} (\mathbf{L}^{-\top} \otimes \mathbf{L}^{-1}) + (\mathbf{I} \otimes \mathbf{L}^{-\top} \mathbf{L}^{-1})\right] \vect(\diff \mathbf{L}),
    \end{equation}
    which implies that the Fisher information matrix $\mathbf{F}$ takes the form
    \begin{align}
        \mathbf{F} &= -\mathbb{E}_{p}\left[\frac{\partial^2}{\partial \vect{\mathbf{L}} \partial \vect{\mathbf{L}}^\top} \log p(\mathbf{x}; \mathbf{L})\right] \nonumber \\
        &= \mathbf{C} (\mathbf{L}^{-\top} \otimes \mathbf{L}^{-1}) + (\mathbf{I} \otimes \mathbf{L}^{-\top} \mathbf{L}^{-1}) \nonumber \\
        &= (\mathbf{I} \otimes \mathbf{L}^{-\top}) (\mathbf{I} + \mathbf{C}) (\mathbf{I} \otimes \mathbf{L}^{-1}).
    \end{align}

    By rearranging the terms of \cref{eq:natgrad_formal}, we have that the natural gradient may be found without explicitly computing the inverse $\mathbf{F}$ by solving the equation
    \begin{equation}
        \mathbf{F} \tilde{\mathbf g} = \mathbf g \label{eq:natgrad_eq}
    \end{equation}
    for $\tilde{\mathbf g}$, under the constraint that the latter is the vectorisation of a lower triangular. Before we proceed, let us introduce the matrix form of the natural gradient vector as $\tilde{\mathbf G} = \unvec(\tilde{\mathbf g})$. Then we compute
    \begin{align}
        \mathbf{F} \tilde{\mathbf g} &= (\mathbf{I} \otimes \mathbf{L}^{-\top}) (\mathbf{I} + \mathbf{C}) (\mathbf{I} \otimes \mathbf{L}^{-1}) \tilde{\mathbf g} \nonumber \\
        &= (\mathbf{I} \otimes \mathbf{L}^{-\top}) (\mathbf{I} + \mathbf{C}) \vect(\mathbf{L}^{-1} \tilde{\mathbf G}) \nonumber \\
        &= (\mathbf{I} \otimes \mathbf{L}^{-\top}) \vect(\mathbf{L}^{-1} \tilde{\mathbf G} + (\mathbf{L}^{-1} \tilde{\mathbf G})^\top) \nonumber \\
        &= \vect(\mathbf{L}^{-\top} \mathbf{L}^{-1} \tilde{\mathbf G} + \mathbf{L}^{-\top} (\mathbf{L}^{-1} \tilde{\mathbf G})^\top).
    \end{align}
    Then, \cref{eq:natgrad_eq} is equivalent to
    \begin{equation}
        \mathbf{L}^{-\top} \mathbf{L}^{-1} \tilde{\mathbf G} + \mathbf{L}^{-\top} (\mathbf{L}^{-1} \tilde{\mathbf G})^\top = \mathbf{A} \mathbf{L} - \mathbf{L}^{-\top}.
    \end{equation}
    By left-multiplying both sides by $\mathbf{L}^\top$, we obtain
    \begin{align}
        \mathbf{L}^{-1} \tilde{\mathbf G} + (\mathbf{L}^{-1} \tilde{\mathbf G})^\top &= \mathbf{L}^\top \mathbf{A} \mathbf{L} - \mathbf{I} \nonumber \\
        &= \tril(\mathbf{L}^\top \mathbf{A} \mathbf{L}) + \triu(\mathbf{L}^\top \mathbf{A} \mathbf{L}) - \diag(\mathbf{L}^\top \mathbf{A} \mathbf{L}) - \mathbf{I} \nonumber \\
        &= \tril(\mathbf{L}^\top \mathbf{A} \mathbf{L}) - \frac{1}{2}(\mathbf{I} + \diag(\mathbf{L}^\top \mathbf{A} \mathbf{L})) \nonumber \\
        &\qquad\qquad + \left[\tril(\mathbf{L}^\top \mathbf{A} \mathbf{L}) - \frac{1}{2}(\mathbf{I} + \diag(\mathbf{L}^\top \mathbf{A} \mathbf{L}))\right]^\top,
    \end{align}
    which implies that
    \begin{equation}
        \mathbf{L}^{-1} \tilde{\mathbf G} = \tril(\mathbf{L}^\top \mathbf{A} \mathbf{L}) - \frac{1}{2}(\mathbf{I} + \diag(\mathbf{L}^\top \mathbf{A} \mathbf{L})).
    \end{equation}
    Finally, by left-multiplying both sides by $\mathbf{L}$, we find that
    \begin{equation}
        \tilde{\mathbf G} = \mathbf{L} \left[\tril(\mathbf{L}^\top \mathbf{A} \mathbf{L}) - \frac{1}{2}(\mathbf{I} + \diag(\mathbf{L}^\top \mathbf{A} \mathbf{L}))\right].
    \end{equation}
\end{proof}

Notice that a similar procedure as the one above can be used to derive the natural gradient of the loss $\ell_\mathbf{A}$ with respect to different parameterisations of the covariance matrix of $p(\mathbf{x})$. This includes the marginal parameterisation $p(\mathbf{x}; \mathbf{T}) = \mathcal{N}(\mathbf{0}, \mathbf{T})$, as well as a general square root parameterisation $p(\mathbf{x}; \mathbf{B}) = \mathcal{N}(\mathbf{0}, \mathbf{B}\mathbf{B})$. In particular, it can be shown that, when using the marginal parameterisation, natural gradient descent with unit step size recovers the well-known Newtonian iteration for computing the inverse of a matrix \citep{ben1965iterative}.

\section{PROOF OF PROPOSITION \texorpdfstring{\ref{prop:grad}}{2}}\label{app:grad}
All gradients and differentials below are taken with respect to $\bm\xi=\{\bm\theta,\mathbf Z,\tilde{\mathbf m},\tilde{\mathbf S}\}$.
We show that the first differentials of $\tilde{\mathcal L}_{\text{lsvgp}}$ and $\tilde{\mathcal L}_{\text{rsvgp}}$ are equal when $\mathbf T=\operatorname{stopgrad}(\tilde{\mathbf K}^{-1})$, from which the gradient equivalence follows.
For ease of reference, we start by restating the SVGP ELBO \eqref{eq:elbo}, the preconditioned L-SVGP bound (\cref{sec:svgp}), and the preconditioned R-SVGP bound (\cref{sec:preconditioning}).

The general SVGP ELBO is
\begin{equation}
    \mathcal{L} = \sum_{n = 1}^N \underbrace{\mathbb{E}_{\mathcal{N}(\mu_n, \sigma_n^2)} \left[\log p(y_n | f(\mathbf{x}_n))\right]}_{(a)} - \underbrace{\KL\left[q(\mathbf{u}) || p(\mathbf{u}) \right]}_{(b)}.
\end{equation}
For L-SVGP, substituting $\mathbf{P}^{(\mathrm L)} = \tilde{\mathbf{K}^{-1}}$, $\tilde{\mathcal{L}}_\text{lsvgp}$ is defined by plugging the following expressions into the ELBO:
\begin{align}
    \mu_n &= \mathbf{k}_{n \mathbf{u}} \mathbf{P}^{(\mathrm L)} \tilde{\mathbf{m}}, \nonumber \\
    \sigma_n^2 &= k_{nn} - \mathbf{k}_{n \mathbf{u}} \mathbf{P}^{(\mathrm L)} \mathbf{k}_{\mathbf{u} n}, \nonumber \\
    \KL\left[q(\mathbf{u}) || p(\mathbf{u})\right] &= \frac{1}{2}\left(-\textcolor{blue}{\text{tr}(\mathbf{P}^{(\mathrm L)} \mathbf{K}_{\mathbf{u} \mathbf{u}})} + \textcolor{blue}{\tilde{\mathbf{m}}^\top \mathbf{P}^{(\mathrm L)} \mathbf{K}_{\mathbf{u} \mathbf{u}} \mathbf{P}^{(\mathrm L)} \tilde{\mathbf{m}}} + \log |\tilde{\mathbf{K}}| - \log |\tilde{\mathbf{S}}|\right) =: \KL_{\text{lsvgp}}. \label{eq:kl_lsvgp_app}
\end{align}
For R-SVGP, $\tilde{\mathcal{L}}_\text{rsvgp}$ is defined by plugging the following expressions into the ELBO:
\begin{align}
  \mu_n &= \mathbf k_{n\mathbf u}\mathbf P^{(\mathrm R)}\tilde{\mathbf m}, \nonumber \\
  \sigma_n^2 &= k_{nn}-\mathbf k_{n\mathbf u}\mathbf P^{(\mathrm R)}\mathbf k_{\mathbf u n}, \nonumber \\
  \KL\left[q(\mathbf u)\,\|\,p(\mathbf u)\right] &\leq \frac{1}{2}\left(-\textcolor{blue}{\operatorname{tr}(\mathbf P^{(\mathrm R)}\mathbf K_{\mathbf u\mathbf u})}+\operatorname{tr}(\tilde{\mathbf K}\mathbf T)-M + \textcolor{blue}{\tilde{\mathbf m}^\top \mathbf P^{(\mathrm R)}\mathbf K_{\mathbf u\mathbf u}\mathbf P^{(\mathrm R)}\tilde{\mathbf m}}-\log|\mathbf T|-\log|\tilde{\mathbf S}|\right) =: \KL_{\text{rsvgp}}. \label{eq:kl_rsvgp_app}
\end{align}

Consider the preconditioners $\mathbf P^{(\mathrm L)}=\tilde{\mathbf K}^{-1}$ and $\mathbf P^{(\mathrm R)}=2\mathbf T-\mathbf T\tilde{\mathbf K}\mathbf T$.
Note that $\mathbf{T}$ does not depend on $\bm\xi$, i.e., $\mathrm{d}\mathbf T = 0$.
Their differentials are
\begin{align}
    \mathrm{d}\mathbf P^{(\mathrm L)} &= \mathrm{d}(\tilde{\mathbf{K}}^{-1}) = -\tilde{\mathbf K}^{-1}(\mathrm{d}\tilde{\mathbf K})\tilde{\mathbf K}^{-1}, \nonumber \\
    \mathrm{d}\mathbf P^{(\mathrm R)} &= \mathrm{d}(2\mathbf T-\mathbf T\tilde{\mathbf K}\mathbf T) = -\mathbf T(\mathrm{d}\tilde{\mathbf K})\mathbf T.
\end{align}
Hence, we have
\begin{equation}
    \mathbf{P}^{(\mathrm L)} = \mathbf{P}^{(\mathrm R)} \text{ and } \mathrm{d}\mathbf P^{(\mathrm L)} = \mathrm{d}\mathbf P^{(\mathrm R)} \quad\text{when}\quad \mathbf{T} = \mathrm{stopgrad}(\tilde{\mathbf{K}}^{-1}), \label{eq:preconditioner_eq}
\end{equation}
which is the key observation for the rest of the proof.

Consider the variational expectation terms $(a)$ in the ELBO.
These depend on $\bm\xi$ only through $\mu_n$ and $\sigma_n^2$.
In turn, $\mu_n$ and $\sigma_n^2$ differ between L-SVGP and R-SVGP only through $\mathbf P^{(\mathrm L)}$ and $\mathbf P^{(\mathrm R)}$, respectively.
Hence, by \eqref{eq:preconditioner_eq}, the differentials of the variational expectation terms match when $\mathbf T=\operatorname{stopgrad}(\tilde{\mathbf K}^{-1})$.

Next, consider the KL term $(b)$ in the ELBO, denoted by $\KL_{\text{lsvgp}}$ \eqref{eq:kl_lsvgp_app} and $\KL_{\text{rsvgp}}$ \eqref{eq:kl_rsvgp_app} for L-SVGP and R-SVGP, respectively.
The terms coloured in \textcolor{blue}{blue} in their expressions also differ between L-SVGP and R-SVGP only through $\mathbf P^{(\mathrm L)}$ and $\mathbf P^{(\mathrm R)}$.
As a result, by \eqref{eq:preconditioner_eq}, the differentials of these terms match when $\mathbf T=\operatorname{stopgrad}(\tilde{\mathbf K}^{-1})$.
Excluding multiplicative constants and the terms in common between $\KL_{\text{lsvgp}}$ and $\KL_{\text{rsvgp}}$ (i.e., $\log|\tilde{\mathbf S}|$), the differentials of the remaining terms are given by
\begin{equation}
    \mathrm{d} \log|\tilde{\mathbf{K}}| = \mathrm{tr}(\tilde{\mathbf{K}}^{-1} \mathrm{d}(\tilde{\mathbf{K}}))
\end{equation}
for L-SVGP, and
\begin{equation}
    \mathrm{d}\left(\mathrm{tr}(\tilde{\mathbf{K}}\mathbf{T}) - M - \log|\mathbf{T}|\right) = \mathrm{tr}\left(\mathbf{T} \mathrm{d}\tilde{\mathbf{K}}\right)
\end{equation}
for R-SVGP, which match when $\mathbf{T} = \operatorname{stopgrad}(\tilde{\mathbf{K}}^{-1})$.
Therefore, the differentials of the KL terms, and in turn of $\tilde{\mathcal L}_{\text{lsvgp}}$ and $\tilde{\mathcal L}_{\text{rsvgp}}$, match when $\mathbf T=\operatorname{stopgrad}(\tilde{\mathbf K}^{-1})$, as desired.

Lastly, the same is not true for the naive preconditioner $\mathbf{P} = \mathbf{T}$.
In particular, the second part of \cref{eq:preconditioner_eq} ceases to hold when $\mathbf{T} = \operatorname{stopgrad}(\tilde{\mathbf{K}}^{-1})$, as $\mathrm{d} \mathbf{P} = \mathrm{d} \mathbf{T} = 0$.

\section{PROOF OF PROPOSITION \texorpdfstring{\ref{prop:lb}}{3}}\label{app:lb}
\begin{proof}
    As in the statement of \cref{prop:lb}, we use the shorthand $\tilde{\mathbf{K}}_{-} = \mathbf{K}_{\mathbf{u}\mathbf{u}} + \tilde{\mathbf{S}}'$, so that $\tilde{\mathbf{K}} = \tilde{\mathbf{K}}_{-} + \sigma^2 \mathbf{I}$. Then, proving the inequality
    \begin{align}
        \sigma^{2(\text{L})}_n  &= k_{nn} - \mathbf{k}_{n \mathbf{u}} \tilde{\mathbf{K}}^{-1} \mathbf{k}_{\mathbf{u} n} \nonumber  \\
        &\geq k_{nn} - \frac{1}{\sigma^2} \left(\mathbf{k}_{n \mathbf{u}} \mathbf{k}_{\mathbf{u} n} - 2 \mathbf{k}_{n \mathbf{u}} \mathbf{T} \tilde{\mathbf{K}}_{-} \mathbf{k}_{\mathbf{u} n} + \mathbf{k}_{n \mathbf{u}} \mathbf{T} \tilde{\mathbf{K}}_{-} \tilde{\mathbf{K}} \mathbf{T}\mathbf{k}_{\mathbf{u} n} \right) \nonumber \\
        &= L_n
    \end{align}
    is equivalent to showing that
    \begin{equation}
        \frac{1}{\sigma^2} \left(\mathbf{k}_{n \mathbf{u}} \mathbf{k}_{\mathbf{u} n} - 2 \mathbf{k}_{n \mathbf{u}} \mathbf{T} \tilde{\mathbf{K}}_{-} \mathbf{k}_{\mathbf{u} n} + \mathbf{k}_{n \mathbf{u}} \mathbf{T} \tilde{\mathbf{K}}_{-} \tilde{\mathbf{K}} \mathbf{T}\mathbf{k}_{\mathbf{u} n} \right) \geq \mathbf{k}_{n \mathbf{u}} \tilde{\mathbf{K}}^{-1} \mathbf{k}_{\mathbf{u} n}.
    \end{equation}
    We write $\mathbf{T} \mathbf{k}_{\mathbf{u} n} = \tilde{\mathbf{K}}^{-1} \mathbf{k}_{\mathbf{u} n} + \bm{\delta}$, where $\bm{\delta}$ is the discrepancy due to the error with which $\mathbf{T}$ approximates $\tilde{\mathbf{K}}^{-1}$.
    Notice that both $\tilde{\mathbf{K}}_{-}$ and $\tilde{\mathbf{K}}$ are symmetric PD.
    Moreover, since $\mathbf{T}$ is parameterised via its Cholesky factor $\mathbf{L}$ with $\mathbf{T} = \mathbf{L}\mathbf{L}^\top$, $\mathbf{T}$ is also symmetric PD.
    Then, we have
    \begin{align}
        &\frac{1}{\sigma^2} \left(\mathbf{k}_{n \mathbf{u}} \mathbf{k}_{\mathbf{u} n} - 2 \mathbf{k}_{n \mathbf{u}} \mathbf{T} \tilde{\mathbf{K}}_{-} \mathbf{k}_{\mathbf{u} n} + \mathbf{k}_{n \mathbf{u}} \mathbf{T} \tilde{\mathbf{K}}_{-} \tilde{\mathbf{K}} \mathbf{T}\mathbf{k}_{\mathbf{u} n} \right) \nonumber \\
        &\qquad = \frac{1}{\sigma^2} \left(\mathbf{k}_{n \mathbf{u}} \mathbf{k}_{\mathbf{u} n} - 2 (\tilde{\mathbf{K}}^{-1} \mathbf{k}_{\mathbf{u} n} + \bm{\delta})^\top \tilde{\mathbf{K}}_{-} \mathbf{k}_{\mathbf{u} n} + (\tilde{\mathbf{K}}^{-1} \mathbf{k}_{\mathbf{u} n} + \bm{\delta})^\top \tilde{\mathbf{K}}_{-} \tilde{\mathbf{K}} (\tilde{\mathbf{K}}^{-1} \mathbf{k}_{\mathbf{u} n} + \bm{\delta}) \right) \nonumber \\
        &\qquad = \frac{1}{\sigma^2} \left(\mathbf{k}_{n \mathbf{u}} \mathbf{k}_{\mathbf{u} n} - \cancel{2} \mathbf{k}_{n \mathbf{u}} \tilde{\mathbf{K}}^{-1} \tilde{\mathbf{K}}_{-} \mathbf{k}_{\mathbf{u} n} \cancel{- 2 \bm{\delta}^\top \tilde{\mathbf{K}}_{-} \mathbf{k}_{\mathbf{u} n}} \right. \nonumber \\ 
        &\qquad\qquad\qquad\qquad \left.+ \cancel{\mathbf{k}_{n \mathbf{u}} \tilde{\mathbf{K}}^{-1} \tilde{\mathbf{K}}_{-} \mathbf{k}_{\mathbf{u} n}} \cancel{+ 2 \bm{\delta}^\top \tilde{\mathbf{K}}_{-} \mathbf{k}_{\mathbf{u} n}} + \bm{\delta}^\top \tilde{\mathbf{K}}_{-} \tilde{\mathbf{K}} \bm{\delta} \right) \nonumber \\
        &\qquad = \frac{1}{\sigma^2} \left(\mathbf{k}_{n \mathbf{u}} \mathbf{k}_{\mathbf{u} n} - \mathbf{k}_{n \mathbf{u}} \tilde{\mathbf{K}}^{-1} (\tilde{\mathbf{K}} - \sigma^2 \mathbf{I}) \mathbf{k}_{\mathbf{u} n} + \bm{\delta}^\top \tilde{\mathbf{K}}_{-} \tilde{\mathbf{K}} \bm{\delta}\right) \nonumber \\
        &\qquad = \mathbf{k}_{n \mathbf{u}} \tilde{\mathbf{K}}^{-1} \mathbf{k}_{\mathbf{u} n} + \frac{1}{\sigma^2}\bm{\delta}^\top \tilde{\mathbf{K}}_{-} \tilde{\mathbf{K}} \bm{\delta} \nonumber \\
        &\qquad \geq \mathbf{k}_{n \mathbf{u}} \tilde{\mathbf{K}}^{-1} \mathbf{k}_{\mathbf{u} n},
    \end{align}
    with equality when $\bm{\delta} = \mathbf{0}$, which is implied by $\mathbf{T} = \tilde{\mathbf{K}}^{-1}$.
    The last inequality follows from the fact that the product of $\tilde{\mathbf{K}}_{-}$ and $\tilde{\mathbf{K}}$ is PSD, as it is the sum of two PSD matrices, i.e.,
    \begin{equation*}
        \tilde{\mathbf{K}}_{-} \tilde{\mathbf{K}} = \tilde{\mathbf{K}}_{-} (\tilde{\mathbf{K}}_{-} + \sigma^2 \mathbf{I}) = \tilde{\mathbf{K}}_{-}^2 + \sigma^2 \tilde{\mathbf{K}}_{-} \succeq \mathbf{0}.
    \end{equation*}
\end{proof}

As already discussed in \cref{sec:stopping}, the upper bound $U_n$ in \cref{eq:lsvgp_ub} and the lower bound $L_n$ in \cref{eq:lsvgp_lb} that we proved above can be monitored to choose the number of NG optimisation steps for $\mathbf{T}$. In particular, by subtracting $U_n$ and $L_n$, we find the quantity
\begin{align}
    G_n &= U_n - L_n \nonumber  \\
    &= k_{n \mathbf{u}} \mathbf{T} \tilde{\mathbf{K}}\mathbf{T}\mathbf{k}_{\mathbf{u}n} - 2 \mathbf{k}_{n\mathbf{u}}\mathbf{T}\mathbf{k}_{\mathbf{u}n} + \frac{1}{\sigma^2} \left(\mathbf{k}_{n \mathbf{u}} \mathbf{k}_{\mathbf{u} n} - 2 \mathbf{k}_{n \mathbf{u}} \mathbf{T} \tilde{\mathbf{K}}_{-} \mathbf{k}_{\mathbf{u} n} + \mathbf{k}_{n \mathbf{u}} \mathbf{T} \tilde{\mathbf{K}}_{-} \tilde{\mathbf{K}} \mathbf{T}\mathbf{k}_{\mathbf{u} n} \right) \nonumber \\
    &= \frac{1}{\sigma^2} \left(\mathbf{k}_{n \mathbf{u}} \mathbf{k}_{\mathbf{u} n} - 2 \mathbf{k}_{n \mathbf{u}} \mathbf{T} (\tilde{\mathbf{K}}_{-} + \sigma^2 \mathbf{I}) \mathbf{k}_{\mathbf{u} n} + \mathbf{k}_{n \mathbf{u}} \mathbf{T} (\tilde{\mathbf{K}}_{-} + \sigma^2 \mathbf{I}) \tilde{\mathbf{K}} \mathbf{T}\mathbf{k}_{\mathbf{u} n} \right) \nonumber \\
    &= \frac{1}{\sigma^2} ||\mathbf{k}_{\mathbf{u}n} - \tilde{\mathbf{K}} \mathbf{T}\mathbf{k}_{\mathbf{u} n}||^2 \nonumber \\
    &= \frac{1}{\sigma^2} ||(\mathbf{I} - \tilde{\mathbf{K}} \mathbf{T}) \mathbf{k}_{\mathbf{u}n}||^2.
\end{align}
By optimising $\mathbf{T}$ until $G_n < \epsilon'$, one can ensure that the R-SVGP predictive variance in \cref{eq:lsvgp_ub} is at most $\epsilon'$ larger than the L-SVGP predictive variance $\sigma^{2(\text{L})}_n$ in \cref{eq:pv}. Moreover, the only part of the SVGP ELBO in \cref{eq:elbo} that depends directly on the predictive variance $\sigma_n^2$ of the parameterisation is the expectation component, which, in the case of Gaussian likelihoods, is given by
\begin{equation}
    \sum_{n=1}^N \mathbb{E}_{\mathcal{N}(\mu_n, \sigma_n^2)} \left[\log p(y_n | f(\mathbf{x}_n)) \right] = -\frac{1}{2 \sigma^2_{\text{obs}}} \sum_{n=1}^N \sigma^2_n + c_3,
\end{equation}
where $\sigma^2_\text{obs}$ is the likelihood variance and $c_3$ is a constant that does not depend on $\sigma^2_n$. Therefore, as mentioned in the main text, optimising $\mathbf{T}$ until the stopping criterion
\begin{equation}
    \sum_{n=1}^N G_n \leq 2 \sigma^2_\text{obs} \epsilon
\end{equation}
is reached ensures that the slack in the ELBO of the R-SVGP parameterisation introduced by upper bounding the predictive variance $\sigma^{2(\text{L})}_n$ of the L-SVGP parameterisation is at most $\epsilon$.

The careful reader may have noticed that the stopping criterion derived above resembles a commonly used approach for iterative GP approximations \citep{mackay1997efficient,davies2015effective}. In particular, iterative frameworks for solving GPs approximate key quantities of exact GPs, such as their predictive mean
\begin{equation}
    \mu_* = \mathbf{k}_{* \mathbf{u}} \mathbf{K}_{\mathbf{f} \mathbf{f}}^{-1} \mathbf{y},
\end{equation}
by using iterative solvers (e.g., conjugate gradients) for the linear system $\mathbf{K}_{\mathbf{f}\mathbf{f}} \mathbf{x} = \mathbf{y}$, until a stopping criterion is reached. A common criterion \citep{mackay1997efficient} monitors the slack due to the approximation in the quadratic term of the true GP log marginal likelihood, i.e.,
\begin{equation}
    -\frac{1}{2} \mathbf{y}^\top \mathbf{K}_{\mathbf{f} \mathbf{f}}^{-1} \mathbf{y}, \label{eq:quad_lml}
\end{equation}
by employing upper and lower bounds on \cref{eq:quad_lml}, which resemble the bounds $U_n$ and $L_n$, with $\mathbf{k}_{\mathbf{u} n}$ replaced by $\mathbf{y}$.

\section{DERIVATION OF ALTERNATIVE INVERSE-FREE BOUND}\label{app:alternative_rsvgp}
As in \cref{sec:svgp}, start from M-SVGP and reparameterise the variational mean and covariance as $\mathbf{m} = \mathbf{K}_{\mathbf{u}\mathbf{u}} \mathbf{R} \tilde{\mathbf{m}}$ and $\mathbf{S} = \mathbf{K}_{\mathbf{u}\mathbf{u}} \mathbf{R} \tilde{\mathbf{S}} \mathbf{R}^\top \mathbf{K}_{\mathbf{u}\mathbf{u}}$, respectively, where $\tilde{\mathbf{S}}$ is PD and $\mathbf{R}$ is lower triangular with positive diagonal.
By plugging these choices into the expressions for the M-SVGP predictive mean and variance, this yields
\begin{equation*}
    \mu_n = \mathbf{k}_{n \mathbf{u}} \mathbf{R} \tilde{\mathbf{m}}, \quad \sigma_n^2 = k_{nn} - \mathbf{k}_{n \mathbf{u}} (\mathbf{K}_{\mathbf{u}\mathbf{u}}^{-1} - \mathbf{R}\tilde{\mathbf{S}}\mathbf{R}^\top) \mathbf{k}_{\mathbf{u} n}.
\end{equation*}

The predictive variance $\sigma_n^2$ above admits an inverse-free upper bound that depends on the auxiliary parameter $\mathbf{R}$.
In particular, since $\mathbf{K}_{\mathbf{u}\mathbf{u}}$ is PD, we have
\begin{equation*}
    (\mathbf{R}\mathbf{R}^\top - \mathbf{K}_{\mathbf{u}\mathbf{u}}^{-1}) \mathbf{K}_{\mathbf{u}\mathbf{u}} (\mathbf{R}\mathbf{R}^\top - \mathbf{K}_{\mathbf{u}\mathbf{u}}^{-1}) = \mathbf{R}\mathbf{R}^\top \mathbf{K}_{\mathbf{u}\mathbf{u}} \mathbf{R}\mathbf{R}^\top - 2 \mathbf{R}\mathbf{R}^\top + \mathbf{K}_{\mathbf{u}\mathbf{u}}^{-1} \succeq \mathbf{0},
\end{equation*}
which implies that
\begin{equation*}
    -\mathbf{K}_{\mathbf{u}\mathbf{u}}^{-1} \preceq \mathbf{R}\mathbf{R}^\top \mathbf{K}_{\mathbf{u}\mathbf{u}} \mathbf{R}\mathbf{R}^\top - 2 \mathbf{R}\mathbf{R}^\top.
\end{equation*}
By adding $\mathbf{R}\tilde{\mathbf{S}}\mathbf{R}^\top$ to both sides and plugging the resulting inequality into the expression for $\sigma_n^2$, we find that
\begin{equation*}
    \sigma_n^2 \leq k_{nn} + \mathbf{k}_{n \mathbf{u}} \mathbf{R} (\mathbf{R}^\top\mathbf{K}_{\mathbf{u}\mathbf{u}}\mathbf{R} - 2\mathbf{I} + \tilde{\mathbf{S}})\mathbf{R}^\top \mathbf{k}_{\mathbf{u} n},
\end{equation*}
with equality when $\mathbf{R} = \chol(\mathbf{K}_{\mathbf{u}\mathbf{u}}^{-1})$.
As mentioned in the main text, working backwards, the upper bound above is exactly the predictive variance induced by the alternative covariance choice $\mathbf{S} = \mathbf{K}_{\mathbf{u}\mathbf{u}} + \mathbf{K}_{\mathbf{u}\mathbf{u}}\mathbf{R}(\mathbf{R}^\top \mathbf{K}_{\mathbf{u}\mathbf{u}} \mathbf{R} - 2 \mathbf{I} + \tilde{\mathbf{S}})\mathbf{R}^\top\mathbf{K}_{\mathbf{u}\mathbf{u}}$.
This is easily checked by plugging such $\mathbf{S}$ into the expression for the M-SVGP predictive variance.

We now derive an upper bound on the resulting term $\KL[q(\mathbf{u}) || p(\mathbf{u})]$.
To start, define the shorthand $\mathbf{A} = \mathbf{R}^\top \mathbf{K}_{\mathbf{u}\mathbf{u}} \mathbf{R}$, and rewrite $\mathbf{S}$ as
\begin{equation*}
    \mathbf{S} = \mathbf{K}_{\mathbf{u}\mathbf{u}} \mathbf{R} \left(\mathbf{A} + \mathbf{A}^{-1} - 2\mathbf{I} + \tilde{\mathbf{S}}\right) \mathbf{R}^\top \mathbf{K}_{\mathbf{u}\mathbf{u}} = \mathbf{K}_{\mathbf{u}\mathbf{u}} \mathbf{R} \mathbf{B} \mathbf{R}^\top \mathbf{K}_{\mathbf{u}\mathbf{u}},
\end{equation*}
where $\mathbf{B} = \mathbf{A} + \mathbf{A}^{-1} - 2\mathbf{I} + \tilde{\mathbf{S}}$.
By plugging the expressions for $\mathbf{m}$ and $\mathbf{S}$ into the M-SVGP KL term and rearranging, we find that
\begin{align*}
    \KL\left[q(\mathbf{u}) || p(\mathbf{u}) \right] &= \frac{1}{2}\left(\text{tr}(\mathbf{K}_{\mathbf{u} \mathbf{u}}^{-1} \mathbf{S}) + \mathbf{m}^\top \mathbf{K}_{\mathbf{u} \mathbf{u}}^{-1} \mathbf{m} - M + \log |\mathbf{K}_{\mathbf{u} \mathbf{u}}| - \log |\mathbf{S}|\right) \\
    &= \frac{1}{2}\left(\text{tr}(\mathbf{R}\mathbf{B}\mathbf{R}^\top \mathbf{K}_{\mathbf{u} \mathbf{u}}) + \tilde{\mathbf{m}}^\top \mathbf{A} \tilde{\mathbf{m}} - M - \log |\mathbf{A}| - \log |\mathbf{B}| \right) \\
    &= \frac{1}{2}\left(\tr\left(\left(\mathbf{A} + \mathbf{A}^{-1} - 2\mathbf{I} + \tilde{\mathbf{S}}\right)\mathbf{A}\right) - \log|\mathbf{B}| + \tilde{\mathbf{m}}^\top \mathbf{A} \tilde{\mathbf{m}} - M - \log|\mathbf{A}|\right) \\
    &= \frac{1}{2}\left(\tr\left(\left(\mathbf{A} + \mathbf{A}^{-1} - 3\mathbf{I} + \tilde{\mathbf{S}}\right)\mathbf{A}\right) - \log|\mathbf{B}| + \tilde{\mathbf{m}}^\top \mathbf{A} \tilde{\mathbf{m}} + \tr\left(\mathbf{A} \right) - M - \log|\mathbf{A}|\right) \\
    &= \frac{1}{2}\left(\tr\left(\left(\mathbf{A} - 3\mathbf{I} + \tilde{\mathbf{S}}\right)\mathbf{A}\right) + M - \log|\mathbf{B}| + \tilde{\mathbf{m}}^\top \mathbf{A} \tilde{\mathbf{m}} + \tr\left(\mathbf{A} \right) - M - \log|\mathbf{A}|\right). \\
\end{align*}
By identifying the last three terms as $2 \cdot \KL\left[\mathcal{N}(\mathbf{0}, \mathbf{R}\mathbf{R}^\top) || \mathcal{N}(\mathbf{0}, \mathbf{K}_{\mathbf{u}\mathbf{u}}^{-1})\right]$, we find that
\begin{equation*}
    \KL\left[q(\mathbf{u}) || p(\mathbf{u}) \right] = \frac{1}{2}\left(\tr\left(\left(\mathbf{A} - 3\mathbf{I} + \tilde{\mathbf{S}}\right)\mathbf{A}\right) + M - \log|\mathbf{B}| + \tilde{\mathbf{m}}^\top \mathbf{A} \tilde{\mathbf{m}}\right) + \KL\left[\mathcal{N}(\mathbf{0}, \mathbf{R}\mathbf{R}^\top) || \mathcal{N}(\mathbf{0}, \mathbf{K}_{\mathbf{u}\mathbf{u}}^{-1})\right].
\end{equation*}
Out of all other terms, only $\log|\mathbf{B}|$ requires matrix decompositions. However, since
\begin{equation*}
    \mathbf{A} + \mathbf{A}^{-1} - 2\mathbf{I} = \mathbf{A}^{-1/2}\left(\mathbf{A}^2 + \mathbf{I} -2 \mathbf{A} \right)\mathbf{A}^{-1/2} = \mathbf{A}^{-1/2}(\mathbf{A} - \mathbf{I})^2 \mathbf{A}^{-1/2} \succeq \mathbf{0},
\end{equation*}
we have that $\mathbf{B} \succeq \tilde{\mathbf{S}}$, and thus $\log|\mathbf{B}| \geq \log|\tilde{\mathbf{S}}|$ with equality when $\mathbf{A} = \mathbf{I}$, which is implied by $\mathbf{R} = \chol(\mathbf{K}_{\mathbf{u}\mathbf{u}}^{-1})$.
By plugging this bound into the expression for $\KL[q(\mathbf{u}) || p(\mathbf{u})]$, we find that
\begin{equation*}
    \KL\left[q(\mathbf{u}) || p(\mathbf{u}) \right] \leq \frac{1}{2}\left(\tr\left(\left(\mathbf{A} - 3\mathbf{I} + \tilde{\mathbf{S}}\right)\mathbf{A}\right) + M - \log|\tilde{\mathbf{S}}| + \tilde{\mathbf{m}}^\top \mathbf{A} \tilde{\mathbf{m}}\right) + \KL\left[\mathcal{N}(\mathbf{0}, \mathbf{R}\mathbf{R}^\top) || \mathcal{N}(\mathbf{0}, \mathbf{K}_{\mathbf{u}\mathbf{u}}^{-1})\right],
\end{equation*}
with equality when $\mathbf{R} = \chol(\mathbf{K}_{\mathbf{u}\mathbf{u}}^{-1})$.

Finally, by plugging the expression for $\mathbf{A}$ back into the bound above, we obtain
\begin{align*}
    \KL\left[q(\mathbf{u}) || p(\mathbf{u}) \right] &\leq \frac{1}{2}\left(\tr\left(\left(\tilde{\mathbf{S}} - 3\mathbf{I} + \mathbf{R}^\top \mathbf{K}_{\mathbf{u}\mathbf{u}} \mathbf{R}\right)\mathbf{R}^\top \mathbf{K}_{\mathbf{u}\mathbf{u}} \mathbf{R}\right) + M - \log|\tilde{\mathbf{S}}| + \tilde{\mathbf{m}}^\top \mathbf{R}^\top \mathbf{K}_{\mathbf{u}\mathbf{u}} \mathbf{R} \tilde{\mathbf{m}}\right) \\ 
    &\qquad\qquad + \KL\left[\mathcal{N}(\mathbf{0}, \mathbf{R}\mathbf{R}^\top) || \mathcal{N}(\mathbf{0}, \mathbf{K}_{\mathbf{u}\mathbf{u}}^{-1})\right],
\end{align*}
as stated in the main text.

\section{EXPERIMENTAL DETAILS}\label{sec:experimental_details}
\subsection{Implementation}
We implement our methods in Python using TensorFlow \citep{abadi2016tensorflow}.
W-SVGP is the default SVGP parameterisation for GPflow \citep{matthews2017gpflow}, and we follow its structure to implement R-SVGP and L-SVGP. 
Their extension to deep architectures is based on GPflux \citep{dutordoir2021gpflux}.
Our implementation is available at \href{https://github.com/stefanocortinovis/relaxedgp/}{https://github.com/stefanocortinovis/relaxedgp/}.

The experiments on toy datasets (\cref{sec:efficacy}) were run on an Apple Silicon M4 Pro CPU with 24GB of memory, while all other experiments (\cref{sec:results}) were run on a single NVIDIA V100 GPU with 32GB of memory.

\subsection{Datasets}\label{sec:datasets}
For the experiments presented in \cref{sec:results}, we use datasets from four sources:
\begin{itemize}
    \item \textbf{UCI repository}\footnote{\href{https://archive.ics.uci.edu/datasets}{https://archive.ics.uci.edu/datasets}}: the \textsc{elevators} ($N = 16599$, $D = 18$) and \textsc{kin40k} ($N = 40000$, $D=8$) scalar regression datasets;
    \item \textbf{OpenML}\footnote{\href{https://www.openml.org/search?type=data}{https://www.openml.org/search?type=data}}: the \textsc{banana} ($N=5300$, $D=2$) classification dataset;
    \item \textbf{Keras}\footnote{\href{https://keras.io/api/datasets}{https://keras.io/api/datasets}}: the \textsc{mnist} ($N = 70000$, $D = 784$) classification dataset;
    \item \textbf{Other}: the \textsc{snelson} ($N = 200$, $D = 1$) regression dataset \citep{snelson2005sparse}.
\end{itemize}
For \textsc{snelson} and \textsc{banana} we fit on the full, unnormalised datasets with no held-out split. 
For \textsc{elevators} and \textsc{kin40k} we use a random $90/10$ train/test split and we standardise the data to have zero mean and unit variance feature-wise using train statistics.
For \textsc{mnist} we use the standard $60\mathrm{k}$/{$10\mathrm{k}$} train/test split and scale pixels to $[0,1]$ with no augmentation.

\subsection{Model Setup}
Here, we describe the model setup and initialisation for all the bounds considered (i.e., W-SVGP, L-SVGP, and R-SVGP) in the experiments presented in \cref{sec:results}.
Unless otherwise specified, model parameters shared by all bounds (e.g., kernel variances and likelihood variance) are initialised using the GPflow defaults.

For all the experiments, we use ARD RBF kernels \citep{williams2006gaussian}.
For the multi-class classification experiment on \textsc{mnist} (\cref{sec:generality}), we also include results for the weighted convolutional kernel, as implemented by \citet{vdw2017convolutional}.

Depending on the experiment, we use a suitable GPflow likelihood: Gaussian for regression, Bernoulli for binary classification, and RobustMax for multi-class classification.

Model setup specific to W-SVGP:
\begin{itemize}
    \item The variational mean $\tilde{\mathbf{m}}$ is initialised to $\tilde{\mathbf{m}} = \mathbf{0}$;
    \item The variational covariance $\tilde{\mathbf{S}}$ is parameterised via its Cholesky factor $\tilde{\mathbf{L}}$ with $\tilde{\mathbf{S}} = \tilde{\mathbf{L}} \tilde{\mathbf{L}}^\top$, and $\tilde{\mathbf{L}}$ is initialised to $\tilde{\mathbf{L}} = \mathbf{I}$.  
\end{itemize}

Model setup specific to L-SVGP:
\begin{itemize}
    \item Apart from the efficacy experiments on toy datasets (\cref{sec:efficacy}), we only consider the preconditioned version of the bound (\cref{sec:preconditioning});
    \item The variational mean $\tilde{\mathbf{m}}$ is initialised to $\tilde{\mathbf{m}} = \mathbf{0}$;
    \item The variational covariance $\tilde{\mathbf{S}}$ is parameterised via its diagonal $\tilde{\mathbf{s}}$ with $\tilde{\mathbf{S}} = \mathrm{diag}(\tilde{\mathbf{s}})$, and $\tilde{\mathbf{s}}$ is initialised to $\tilde{\mathbf{s}} = \alpha \mathbf{1}$ with $\alpha = 10^{-4}$. 
\end{itemize}

Model setup specific to R-SVGP:
\begin{itemize}
    \item Apart from the efficacy experiments on toy datasets (\cref{sec:efficacy}), we only consider the preconditioned version of the bound (\cref{sec:preconditioning});
    \item The variational parameters $\tilde{\mathbf{m}}$ and $\tilde{\mathbf{S}}$ are initialised as for L-SVGP;
    \item The matrix $\mathbf{T}$ is parameterised via its Cholesky factor $\mathbf{L}$ with $\mathbf{T} = \mathbf{L} \mathbf{L}^\top$, and $\mathbf{L}$ is initialised to $\mathbf{L} = \beta \mathbf{I}$ with $\beta = 10^{-3}$.
    \item For the efficiency experiments on the UCI datasets (\cref{sec:efficiency}), we estimate the trace terms in the bound using $K = 256$ random probe vectors, as discussed in \cref{sec:practical}.
\end{itemize}

The number $M$ of inducing points and the initialisation of the corresponding inducing location $\mathbf{Z}$ depend on the experiment performed:
\begin{itemize}
    \item Shallow regression/classification on toy datasets (\cref{sec:efficacy}): $M = 10$ for \textsc{snelson}, initialised on a uniform grid spanning the range of the training inputs, and $M = 64$ for \textsc{banana}, initialised using $k$-means\texttt{++} \citep{arthur2007kmeans} on the training inputs;
    \item Shallow regression on UCI datasets (\cref{sec:efficiency}): $M$ ranging from $1000$ to $4000$ for both \textsc{elevators} and \textsc{kin40k}, initialised using $k$-means\texttt{++} on the training inputs;
    \item Deep regression on \textsc{kin40k} (\cref{sec:generality}): $M = 500$ for each layer, initialised using the GPflux default (based on $k$-means\texttt{++});
    \item Multi-class classification on \textsc{mnist} (\cref{sec:generality}): $M = 750$ initialised by sampling uniformly without replacement from the training inputs.
\end{itemize}

As discussed in \cref{sec:generality}, for deep GP experiments, we range the number of layers from $1$ to $3$, with each intermediate layer $\ell$ having $Q^{(\ell)} = D$ outputs, where $D$ is the input dimension of the dataset.
Kernel hyperparameters and inducing locations are shared between outputs within each layer, as per GPflux defaults.
As discussed in \cref{sec:generality}, for deep GPs based on L-SVGP and R-SVGP, we also tie the variational covariance matrices $\tilde{\mathbf{S}}$ within each layer to match the cost of W-SVGP-based models.

\subsection{Training}\label{app:training}
Training details are specific to each experiment, and are described below.

\paragraph{Shallow Regression/Classification on Toy Datasets (\cref{sec:efficacy}).}
We keep the inducing locations $\mathbf{Z}$ fixed during training, use a mini-batch size $B$ equal to the number of inducing points $M$, and perform $10000$ training iterations.
The learning rate of Adam is kept fixed during training to $5 \times 10^{-3}$ for the \textsc{snelson} dataset, and to $1 \times 10^{-2}$ for the \textsc{banana} dataset.
For $\mathbf{T}$-optimisation in R-SVGP, a single NG update with step size $1.0$ is used at every training iteration.

\paragraph{Shallow Regression on UCI Datasets (\cref{sec:efficiency}).}
We use a mini-batch size $B = 100$ and perform $20000$ training iterations.
The learning rate of Adam is kept fixed during training to $5 \times 10^{-3}$.
For $\mathbf{T}$-optimisation in R-SVGP, we take advantage of the training heuristics discussed in \cref{sec:practical}.
In particular, $\mathbf{Z}$ is kept fixed for the first $1000$ training iterations, before being trained jointly with the other parameters using a $\mathbf{Z}$-specific Adam configuration with $\beta_1 = 0.99$ and learning rate starting $1 \times 10^{-3}$ and multiplied by $0.95$ every time a plateau in the training loss is reached for $100$ iterations.
As a result of these heuristics, $\mathbf{T}$ can be optimised using a single NG update with step size $1.0$ at every training iteration.

\paragraph{Deep Regression (\cref{sec:generality}).}
We use a mini-batch size $B = 1000$ and perform $200000$ training iterations.
The learning rate of Adam is initialised to $5 \times 10^{-3}$ and multiplied by $0.95$ every time a plateau in the training loss is reached for $1000$ iterations.
For $\mathbf{T}$-optimisation in R-SVGP, we use the stopping criterion \eqref{eq:stopping_frobenius} in \cref{sec:stopping} with tolerance $\epsilon = 10^{-9}$ and a log-linear step-size schedule \citep{salimbeni2018natural} starting from $1 \times 10^{-5}$ and increasing to $1.0$ over the first $10$ NG updates.

\paragraph{Multi-class Classification (\cref{sec:generality}).}
We follow the training setup of \citet{vdw2017convolutional}.
In particular, we use a mini-batch size $B = 200$ and perform $450000$ training iterations.
The learning rate of Adam is initialised to $10^{-3}$ and annealed using a piecewise-constant schedule with boundaries every $30000$ iterations and values multiplied by $10^{-1/3} \approx 0.464$ at each boundary.
For $\mathbf{T}$-optimisation in R-SVGP, we use the stopping criterion \eqref{eq:stopping_frobenius} in \cref{sec:stopping} with tolerance $\epsilon = 10^{-6}$ and a log-linear step-size schedule starting from $1 \times 10^{-5}$ and increasing to $1.0$ over the first $10$ NG updates.

\section{ADDITIONAL EXPERIMENTS}\label{sec:additional_experiments}
\subsection{Approximate Bound Evaluation via Hutchinson's Method}\label{app:hutchinson}
This section examines the runtime--accuracy trade-off of using Hutchinson's method to estimate the trace terms in the R-SVGP bound, as discussed in \cref{sec:practical}.
To this end, we consider the \textsc{elevators} and \textsc{kin40k} UCI regression datasets, and train R-SVGP using the same setup as in \cref{sec:efficiency} with $M = 2000$ inducing points and batch size $B = 100$, while varying the number of probe vectors $K$ used for trace estimation.
\Cref{fig:hutchinson} tracks the negative log-predictive density (NLPD) against training time on both datasets for different choices of $K$.
\begin{figure*}
    \centering
    \includegraphics[width=1.0\textwidth]{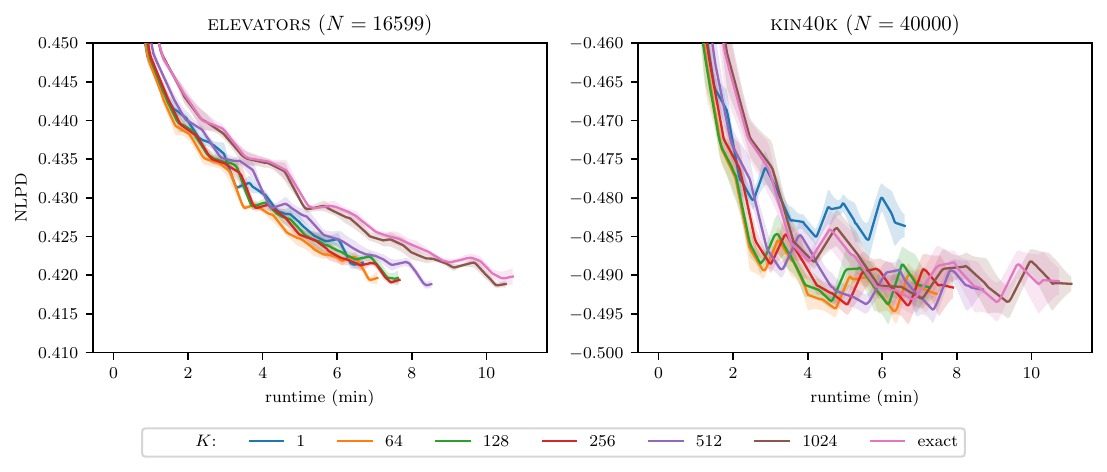}
    \caption{NLPD/runtime on \textsc{elevators} and \textsc{kin40k} datasets for different choices of the number of probes $K$, with $M = 2000$ and $B = 100$. Lines show the mean over 5 seeds; shaded regions indicate $\pm 1$ standard error.}
    \label{fig:hutchinson}
\end{figure*}
As expected, using a single probe vector ($K = 1$) leads to noticeably noisier optimisation, especially on \textsc{kin40k}, and to worse final NLPD.
In contrast, for $K \geq 64$, the final NLPD after training is very similar between runs using approximate and exact trace estimation, while runtime is substantially reduced relative to larger-$K$ runs.
This indicates that Hutchinson's method can provide stable training and good predictive performance with a moderate number of probe vectors in this setting.
For instance, the choice $K = 256$ used in \cref{sec:efficiency} reduces total runtime by approximately $25$--$30\%$ relative to exact trace estimation.

\end{document}